%% file: main_paper.tex
\documentclass{article}

\usepackage{microtype}
\usepackage{graphicx}
\usepackage{subfigure}
\usepackage{booktabs} % for professional tables

\usepackage{hyperref}

\usepackage[accepted]{icml2020}

\icmltitlerunning{CLUB: A Contrastive Log-ratio Upper Bound of Mutual Information}

\input{math_commands.tex}

\newtheorem{corollary}[theorem]{Corollary}

\begin{document}

\twocolumn[
\icmltitle{CLUB: A Contrastive Log-ratio Upper Bound of Mutual Information}


\icmlsetsymbol{equal}{*}

\begin{icmlauthorlist}
\icmlauthor{Pengyu Cheng}{duke}
\icmlauthor{Weituo Hao}{duke}
\icmlauthor{Shuyang Dai}{duke}
\icmlauthor{Jiachang Liu}{duke}
\icmlauthor{Zhe Gan}{ms}
\icmlauthor{Lawrence Carin}{duke}
\end{icmlauthorlist}

\icmlaffiliation{duke}{Department of Electrical and Computer Engineering, Duke University, Durham, North Carolina, USA}
\icmlaffiliation{ms}{Microsoft, Redmond, Washington, USA}

\icmlcorrespondingauthor{Pengyu Cheng}{pengyu.cheng@duke.edu}

\icmlkeywords{Machine Learning, ICML, Mutual Information, Representation Learning, Contrastive Learning}

\vskip 0.3in
]



\printAffiliationsAndNotice{}  % leave blank if no need to mention equal contribution

\begin{abstract}
Mutual information (MI) minimization has gained considerable interests in various machine learning tasks. However, estimating and minimizing MI in high-dimensional spaces remains a challenging problem, especially when only samples, rather than distribution forms, are accessible. Previous works mainly focus on MI lower bound approximation, which is not applicable to MI minimization problems. In this paper, we propose a novel Contrastive Log-ratio Upper Bound (CLUB) of mutual information. We provide a theoretical analysis of the properties of CLUB and its variational approximation. Based on this upper bound, we introduce a MI minimization training scheme and further accelerate it with a negative sampling strategy. Simulation studies on Gaussian distributions show the reliable estimation ability of CLUB. Real-world MI minimization experiments, including domain adaptation and information bottleneck,  demonstrate the effectiveness of the proposed method. The code is at {\small \url{https://github.com/Linear95/CLUB}}.
\end{abstract}

\vspace{-6mm}
\section{Introduction}
\vspace{-1.mm}
% mutual information is important in machine learning
%Measuring dependence between variables is a fundamental problem in machine learning tasks. To quantify the dependence, 

Mutual information (MI) is a fundamental measure of the dependence between two random variables. Mathematically, the definition of MI between variables $\vx$ and $\vy$ is 
\begin{equation}\label{eq:MI-definition}
    \MI(\vx; \vy)  = \bbE_{p(\vx, \vy)} \left[\log \frac{p(\vx, \vy)}{p(\vx) p(\vy)} \right].
\end{equation}
This important tool has been applied in a wide range of scientific fields, including statistics~\citep{granger1994using,jiang2015nonparametric}, bioinformatics~\citep{lachmann2016aracne,zea2016mitos}, robotics~\citep{julian2014mutual,charrow2015information}, and machine learning~\citep{chen2016infogan,alemi2016deep,hjelm2018learning,cheng2020improving}.

% MI minimization is important
%MI minimization has gained considerable attention in disentangled representation learning. 

In machine learning, especially in deep learning frameworks, MI is typically utilized as a criterion or a regularizer in loss functions, to encourage or limit the dependence between variables. MI maximization has been studied extensively in various tasks, \textit{e.g.}, representation learning~\citep{hjelm2018learning,hu2017learning}, generative models~\citep{chen2016infogan}, information  distillation~\citep{ahn2019variational}, and reinforcement learning~\citep{florensa2017stochastic}. Recently, MI minimization has obtained increasing attention for its applications in disentangled representation learning~\cite{chen2018isolating}, style transfer~\citep{kazemi2018unsupervised}, domain adaptation~\citep{gholami2018unsupervised}, fairness~\citep{kamishima2011fairness}, and the information bottleneck~\citep{alemi2016deep}.

% the exact MI value can not be obtain, many approximation not focus on MI upper bound
However, only in a few special cases can one calculate the exact value of mutual information, since the calculation requires closed forms of density functions and a tractable log-density ratio between the joint and marginal distributions. In most machine learning tasks, only samples from the joint distribution are accessible. Therefore, sample-based MI estimation methods have been proposed.
To approximate MI, most previous works focused on lower-bound estimation~\citep{chen2016infogan,belghazi2018mutual,oord2018representation}, which is inconsistent to MI minimization tasks. In contrast, MI upper bound estimation lacks extensive exploration in the literature. Among the existing MI upper bounds, \citet{alemi2016deep} fixes one of the marginal distribution ($p(\vy)$ in \eqref{eq:MI-definition}) to a  standard Gaussian, and obtains a variational upper bound in closed form. However, the Gaussian marginal distribution assumption is unduly strong, which makes the upper bound fail to estimate MI with low bias.
\citet{poole2019variational} points out a leave-one-out upper bound, which provides tighter MI estimation when sample size is large. However, it suffers from high numerical instability in practice when applied to MI minimization models. %However, we show that the leave-one-out upper bound from \citet{poole2019variational} is not numerically stable in practice. Thus, these lower bounds can not be readily used to minimize the MI. 

% our method gives a new upper bound
To overcome the defects of previous MI estimators, we introduce a \textbf{C}ontrastive \textbf{L}og-ratio \textbf{U}pper \textbf{B}ound (CLUB). Specifically, CLUB
bridges
mutual information estimation with contrastive learning~\citep{oord2018representation}, where MI is estimated by the difference of conditional probabilities  between positive and negative sample pairs. Further, we develop a variational form of CLUB (vCLUB) into scenarios where the conditional distribution $p(\vy| \vx)$ is unknown, by approximating $p(\vy| \vx)$ with a neural network. We theoretically prove that, with good variational approximation, vCLUB can either provide reliable MI estimation or remain a valid MI upper bound.
%the approximation still provides an MI upper bound. 
Based on this new bound, we propose an MI minimization algorithm, and further accelerate it via a negative sampling strategy. 
%Simulation study shows that CLUB has a better bias-variance trade-off than other MI estimators. Moreover, we demonstrate that CLUB outperforms other MI minimization methods on two real-world applications: ($i$) information bottleneck, and ($ii$) domain adaptation. 
%
%
%
% The main contributions of this paper are summarized
% as follows.
% %\begin{itemize}
%     %\item 
%     ($i$) We introduce a  Contrastive  Log-ratio Upper Bound (CLUB) of mutual information, which is not only reliable as a mutual information estimator,
%     %is not only tight, \zhe{tight? be careful about this.} 
%     but also trainable in gradient-descent frameworks.  %and extend it by approximating conditional distribution $p(\vy| \vx)$ with variational distributions.
%     %\item 
%     ($ii$) We extend CLUB with variational network approximation, and provide theoretical support to this variational bound.
%     ($iii$)
%     We develop a CLUB-based MI minimization algorithm, and accelerate it with an negative sampling strategy.
%     %\item 
%     ($iv$) We compare CLUB with previous MI estimators on both simulation studies and real-world applications, which demonstrate CLUB is 
%     not only better in the bias-variance estimation trade-off, but also more effective when applied to MI minimization.
%     %\item 
% %    ($iv$) We apply CLUB to information bottleneck and domain adaptation tasks, and demonstrate that our method outperforms other MI estimation approaches on both applications.
% %\end{itemize}
%
%
The main contributions of this paper are summarized
as follows.
\vspace{-2mm}
\begin{itemize}
    \item  We introduce a  Contrastive  Log-ratio Upper Bound (CLUB) of mutual information, which is not only reliable as a mutual information estimator,
    but also trainable in gradient-descent frameworks.  %and extend it by approximating conditional distribution $p(\vy| \vx)$ with variational distributions.
    \vspace{-1.3mm}
    \item  We extend CLUB with a  variational network approximation, and provide theoretical analysis to the good properties of this variational bound.
    \vspace{-1.3mm}
\item    We develop a CLUB-based MI minimization algorithm, and accelerate it with a negative sampling strategy.
    \vspace{-1.3mm}
\item  We compare CLUB with previous MI estimators on both simulation studies and real-world applications, which demonstrate CLUB is     not only better in the bias-variance estimation trade-off, but also more effective when applied to MI minimization.
\vspace{-1mm}
\end{itemize}

\vspace{-2.5mm}
\section{Background}\label{sec:bounds}
%\vspace{-1mm}
Although widely used in numerous applications, mutual information (MI) remains challenging to estimate accurately, when the closed-forms of distributions are unknown or intractable. 
Earlier MI estimation approaches include non-parametric binning~\citep{darbellay1999estimation}, kernel density estimation~\citep{hardle2004nonparametric}, likelihood-ratio estimation~\citep{suzuki2008approximating}, and $K$-nearest neighbor entropy estimation~\citep{kraskov2004estimating}. 
These methods fail to provide reliable approximations when the data dimension increases~\citep{belghazi2018mutual}. Also, the gradient of these estimators is difficult to calculate, which makes them inapplicable to back-propagation frameworks for MI optimization tasks.

To obtain differentiable and scalable MI estimation, recent approaches utilize deep neural networks to construct variational MI estimators. Most of these estimators focus on MI maximization problems, and provide MI lower bounds. Specifically, \citet{barber2003algorithm} replaces the conditional distribution $p(\vy| \vx)$ with an auxiliary distribution $q(\vy| \vx)$, and obtains the Barber-Agakov (BA) bound:
\begin{equation}\label{eq:MI-lower-bound-BA}
 \textstyle\MI_{\text{BA}}:=   \mathrm{H}(\vx) + \bbE_{p(\vx, \vy)}[\log q(\vx| \vy)]  \leq \MI(\vx; \vy) , %\nonumber
\end{equation}
where $\mathrm{H}(\vx)$ is the entropy of variable $\vx$. \citet{belghazi2018mutual} introduces a Mutual Information Neural Estimator (MINE), which treats MI as the Kullback-Leibler (KL) divergence~\citep{kullback1997information} between the joint and marginal distributions, and converts it into the dual representation: %\zhe{add a citation here.}
% \begin{align}
%     \MI(\vx; \vy) =& \, \text{KL}(p(\vx,\vy) \Vert p(\vx) p(\vy)) \\ 
%     \geq & \, \bbE_{p(\vx,\vy)}[f(\vx,\vy)] - \log(\bbE_{p(\vx)p(\vy)} [ e^{f(\vx,\vy)}]), \nonumber
% \end{align}
\begin{equation}
   \textstyle \MI_{\text{MINE}}: =  \bbE_{p(\vx,\vy)}[f(\vx,\vy)] - \log(\bbE_{p(\vx)p(\vy)} [ e^{f(\vx,\vy)}]), \label{eq:mine}
\end{equation}
where $f(\cdot, \cdot)$ is a score function (or, a critic) approximated by a neural network.  Nguyen, Wainwright, and Jordan (NWJ) \citep{nguyen2010estimating} derives another lower bound based on the MI $f$-divergence representation:
\begin{equation}
  \MI_{\text{NWJ}}:=\bbE_{p(\vx,\vy)} [f(\vx,\vy)] - \bbE_{p(\vx)p(\vy)} [ e^{f(\vx, \vy)-1}]. \label{eq:NWJ}
\end{equation}
More recently, based on Noise Contrastive Estimation (NCE)~\citep{gutmann2010noise}, an MI lower bound, called InfoNCE, was introduced in~\citet{oord2018representation}:
\begin{equation}
    \MI_{\text{NCE}}:= \bbE\left[\frac{1}{N} \sum_{i = 1}^N \log \frac{e^{f(\vx_i,\vy_i)}}{\frac{1}{N}\sum_{j=1}^N e^{f(\vx_i, \vy_j)}} \right], \label{eq:NCE}
\end{equation}
where the expectation is over $N$ samples $\{ (\vx_i , \vy_i)\}_{i=1}^N$ drawn from the joint distribution $p(\vx,\vy)$. %\pengyu{add smile depending on experiment}

%When the exact MI values are unobtainable, one alternative solution to MI minimization problem is minimizing MI upper bounds. 
%Although MI minimization are useful in plenty of machine learning tasks, 
Unlike the above MI lower bounds that have been studied extensively,
MI upper bounds are still lacking extensive published exploration. Most existing MI upper bounds require the conditional distribution $p(\vy | \vx)$ to be known. For example, \citet{alemi2016deep} introduces a variational marginal approximation $r(\vy)$ to build a variational upper bound (VUB):
\begin{align}
    \MI(\vx; \vy) = &\, \bbE_{p(\vx, \vy)} [\log \frac{p(\vy| \vx)}{p(\vy)}] \nonumber \\  
    %=&  \bbE_{p(\vx,\vy)} [\log (\frac{p(\vy|\vx)}{q(\vy)} \frac{q(\vy)}{p(\vy)})] \nonumber\\
    = & \, \bbE_{p(\vx,\vy)} [\log \frac{p(\vy|\vx)}{r(\vy)}] - \text{KL}(p(\vy) \Vert r(\vy)) \nonumber\\
    \leq & \, \bbE_{p(\vx,\vy)} [\log \frac{p(\vy|\vx)}{r(\vy)}]= \text{KL}(p(\vy| \vx) \Vert r(\vy)). \label{eq:var-upper-bound}
\end{align}
The inequality is based on the fact that the KL-divergence is always non-negative. To be a good MI estimation, this upper bound requires a well-learned density approximation $r(\vy)$ to $p(\vy)$, so that the difference $\KL(p(\vy) \Vert r(\vy))$ could be small. However, learning a good marginal approximation $r(\vy)$ without any additional information, recognized as the distribution density estimation problem~\citep{magdon1999neural}, is challenging, especially when variable $\vy$ is in a high-dimensional space.
In practice, \citet{alemi2016deep} fixes $r(\vy)$ as a standard normal distribution, $r(\vy) = \calN(\vy | \bm{0}, \bm{\mathrm{I}})$, which results in a high-bias MI estimation.
With $N$ sample pairs $\{(\vx_i, \vy_i)\}_{i=1}^N$, \citet{poole2019variational} replaces $r(\vy)$ with a Monte Carlo approximation $r_i(\vy) = \frac{1}{N-1} \sum_{j \neq i}p(\vy | \vx_j) \approx p(\vy)$ and derives a leave-one-out upper bound (\loout):
\begin{equation}\label{eq:leave-one-out}
    \MI_{\text{L$\bm{1}$Out}}:= \bbE \left[ \frac{1}{N} \sum_{i=1}^N \left[\log \frac{p(\vy_i | \vx_i)}{\frac{1}{N-1} \sum_{j \neq i} p(\vy_i | \vx_j)}\right]\right]. 
\end{equation}
This bound does not require any additional parameters, but highly depends on a sufficient sample size to achieve satisfying Monte Carlo approximation. In practice, {\loout} suffers from numerical instability when applied to real-world MI minimization problems.
%between high-dimensional vectors. \pengyu{judgement to leave one out?}

To compare our method with the aforementioned MI upper bounds in more general scenarios (\textit{i.e.}, $p(\vy | \vx)$ is unknown), we use a neural network $q_\theta(\vy | \vx)$ to approximate $p(\vy| \vx)$, and develop  variational versions of VUB and {\loout}  as :
\begin{align}
  \textstyle  \MI_{\text{vVUB}} &= \bbE_{p(\vx,\vy)} \left[\log \frac{q_\theta(\vy|\vx)}{r(\vy)}\right], \\
\textstyle    \MI_{\text{v\loout}}& = \bbE \left[ \frac{1}{N} \sum_{i=1}^N \left[\log \frac{q_\theta(\vy_i | \vx_i)}{\frac{1}{N-1} \sum_{j \neq i} q_\theta(\vy_i | \vx_j)}\right]\right].
\end{align}
We discuss theoretical properties of these two variational bounds in the Supplementary Material.  In a simulation study (Section~\ref{sec:gaussian-est}), variational {\loout } reaches better performance than previous lower bounds for MI estimation. However, the numerical instability problem remains for variational {\loout } in real-world applications (Section~\ref{sec:domain-adaptation}).
To the best of our knowledge, we  provide the first %\textit
{variational} version of VUB and {\loout } upper bounds, and study their properties on both the theoretical analysis and the empirical performance.

\vspace{-2.mm}
\section{Proposed Method}
\vspace{-1.mm}
%Assume $\vx$, $\vy \sim p(\vx, \vy)$ are two random variables. To minimize $\MI(\vx; \vy)$, we propose a novel sample-based upper bound.
Suppose we have sample pairs $\{(\vx_i, \vy_i)\}_{i=1}^N$ drawn from an unknown or intractable distribution $p(\vx, \vy)$. We aim to derive a upper bound estimator of the mutual information  $\MI(\vx; \vy)$  based on the given samples. In a range of machine learning tasks (\emph{e.g.}, information bottleneck), one of the conditional distributions between variables $\vx$ and $\vy$ (as $p(\vx | \vy)$ or $p(\vy | \vx)$) can be known. To efficiently utilize this additional information, we first derive a  mutual information (MI) upper bound with the assumption that one of the conditional distribution is provided (suppose $p(\vy | \vx)$ is provided, without loss of generality). Then, we extend the bound into more general cases where no  conditional distribution is known. Finally, we develop a MI minimization algorithm based on the derived bound. 

\vspace{-2mm}
\subsection{ CLUB with $p(\vy| \vx)$ Known} 
\vspace{-1.mm}
With the conditional distribution $p(\vy | \vx)$,  our MI Contrastive  Log-ratio  Upper Bound (CLUB) is defined as:
\begin{align}
    \MI_{\text{CLUB}}(\vx ;\vy):  =& \bbE_{p(\vx,\vy)} [\log p(\vy|\vx)] \nonumber\\ &- \bbE_{p(\vx)}\bbE_{p(\vy)} [\log p(\vy| \vx)]. %\nonumber%\label{eq:club-define}.
\end{align}
To show that $\MI_\text{CLUB}(\vx; \vy)$ is an upper bound of $\MI(\vx; \vy)$, we calculate the gap $\Delta$ between them:
\begin{align}
    \Delta := &\MI_\text{CLUB} (\vx;\vy) - \MI(\vx; \vy)\nonumber \\
    = &  \bbE_{p(\vx,\vy)} [\log p(\vy|\vx)]  - \bbE_{p(\vx)}\bbE_{p(\vy)} [\log p(\vy| \vx)]  \nonumber\\ & \ \ \ \ \ \ \ \ \ \ \ \ \ \ \  -  \bbE_{p(\vx, \vy)} \left[\log {p(\vy| \vx)} - \log{p(\vy)}\right] \nonumber \\
    =& \bbE_{p(\vx,\vy)} [\log p(\vy)] - \bbE_{p(\vx)} \bbE_{p(\vy)} [\log p(\vy| \vx)] \nonumber \\
    = & \bbE_{p(\vy)} \left[ \log p(\vy) - \bbE_{p(\vx)}\left[\log p(\vy| \vx)\right] \right].  \label{eq:club-mi-gap}
\end{align}
By the definition of the marginal distribution, we have
$    p(\vy) = \int p(\vy | \vx) p(\vx) \mathrm{d}\vx = \bbE_{p(\vx)} [p(\vy | \vx)] .$
% \begin{equation}\textstyle
%     p(\vy) = \int p(\vy | \vx) p(\vx) \mathrm{d}\vx = \bbE_{p(\vx)} [p(\vy | \vx)] .
% \end{equation}
% Noting that $\log(\cdot)$ is a concave function, by Jensen's Inequality, for any given $\vy$,
% \begin{equation}
%     \log p(\vy) = \log \left(\bbE_{p(\vx)} [ p(\vy| \vx)]\right)  \geq \bbE_{p(\vx)} [ \log p(\vy| \vx)]. %\nonumber
% \end{equation}
%
Note that $\log(\cdot)$ is a concave function, by Jensen's Inequality, we have 
$    \log p(\vy) = \log \left(\bbE_{p(\vx)} [ p(\vy| \vx)]\right)  \geq \bbE_{p(\vx)} [ \log p(\vy| \vx)]$. %\nonumber
%
% Taking expectation with respect to $p(\vy)$, we conclude that the gap $\Delta$ in Eqn.~\eqref{eq:club-mi-gap} is always non-negative.
%
Applying this inequality to equation~\eqref{eq:club-mi-gap}, we conclude that the gap $\Delta$ is always non-negative.
Therefore, $\MI_\text{CLUB} (\vx;\vy)$ is an upper bound of $\MI(\vx; \vy)$. The bound is tight when $p(\vy | \vx)$ has the same value for any $\vx$, which means variables $\vx$ and $\vy$ are independent. Consequently, we summarize the above discussion into the following Theorem~\ref{thm:club}.

\begin{theorem}\label{thm:club}
For two random variables $\vx$ and $\vy$,
\begin{equation}\textstyle
\MI(\vx; \vy) \leq \MI_{\text{CLUB}}(\vx; \vy).
\end{equation}
Equality is achieved if and only if $\vx$ and $\vy$ are independent.
\end{theorem}

With sample pairs $\{(\vx_i, \vy_i)\}_{i=1}^N$, $\MI_{\text{CLUB}}(\vx; \vy)$ has an unbiased estimation as:
\begin{align}
     &\hat{\MI}_{\text{CLUB}} =  \frac{1}{N} \sum_{i = 1}^N \log p(\vy_i | \vx_i ) - \frac{1}{N^2} \sum_{i=1}^N \sum_{j=1}^N \log p(\vy_j | \vx_i )   \nonumber \\
     &= \frac{1}{N^2} \sum_{i=1}^N \sum_{j=1}^N \Big[\log p(\vy_i |\vx_i) - \log {p(\vy_j| \vx_i)} \Big]. \label{eq:club-sample-est}
\end{align}
%\zhe{we need be careful about this. it should be $N(N-1)$, because when $i=j$, the above term is 0, so the actual term is only $N(N-1)$.}
In the estimator $\hat{\MI}_{\text{CLUB}}$,  $\log p(\vy_i | \vx_i )$ provides the conditional log-likelihood of positive sample pair $(\vx_i , \vy_i )$; $\{\log p(\vy_j | \vx_i ) \}_{i \neq j}$ provide the conditional log-likelihood of negative sample pair $(\vx_i , \vy_j)$. The difference between $\log p(\vy_i  | \vx_i )$ and $\log p(\vy_j | \vx_i )$ is the contrastive probability log-ratio between two conditional distributions. Therefore, we name this novel MI upper bound estimator as \textbf{C}ontrastive \textbf{L}og-ratio \textbf{U}pper \textbf{B}ound (CLUB).
Compared with previous MI neural estimators,
CLUB has a simpler form as a linear combination of log-ratios between positive and negative sample pairs. The linear form of log-ratios improves the numerical stability for calculation of CLUB and its gradient, which we  discuss in details in Section~\ref{sec:MI_minimizaiton}.
%\pengyu{need modification}
%These methods take logarithm to the mean of probabilities, 
%[form simple, MI bounded by average difference of log-ratio, mean of log of prob - numerically stable. gradient stable]

%The newly derived upper bound in equation~\eqref{eq:CLUB} calculates the difference between likelihood of positive  pair and the average likelihood of negative pairs, which shares the similar idea with contrastive learning~\citep{gutmann2010noise} \pengyu{ More explanation?}. Therefore, we name equation~\eqref{eq:CLUB} as Contrastive Learning Upper Bound (CLUB) for mutual information.

\vspace{-1.5mm}
\subsection{CLUB with Conditional Distributions Unknown}\label{sec:vCLUB}
\vspace{-1.mm}
When the conditional distributions $p(\vy| \vx)$ or $p(\vx | \vy)$ is provided, the MI can be directly upper-bounded by equation~\eqref{eq:club-sample-est} with samples $\{(\vx_i , \vy_i) \}_{i=1}^N$. Unfortunately, in a large number of machine learning tasks, the conditional relation between variables is unavailable.

To further extend the CLUB estimator into more general scenarios, we  use a variational distribution $q_\theta(\vy | \vx)$ with parameter $\theta$ to  approximate $p(\vy | \vx)$.
%, where $\calQ$ is a variational distribution family.
Consequently, a variational CLUB term (vCLUB) is defined by:
\begin{align}
    \MI_{\text{vCLUB}}(\vx; \vy) := & \bbE_{p(\vx,\vy)} [\log q_\theta(\vy|\vx)] \nonumber\\ &- \bbE_{p(\vx)}\bbE_{p(\vy)} [\log q_\theta(\vy| \vx)] .% \label{eq:var-club-define}.
\end{align}
Similar to the MI upper bound estimator $\hat{\MI}_{\text{CLUB}}$ in \eqref{eq:club-sample-est}, the unbiased estimator for vCLUB  with samples $\{\vx_i , \vy_i \}$ is: 
 \begin{align}
     \hat{\MI}_{\text{vCLUB}} 
     = &\frac{1}{N^2} \sum_{i=1}^N \sum_{j=1}^N \Big[\log q_\theta(\vy_i  |\vx_i ) - \log {q_\theta(\vy_j| \vx_i )} \Big] \nonumber \\ 
     = \frac{1}{N} \sum_{i=1}^N & \Big[ \log q_\theta(\vy_i  |\vx_i ) - \frac{1}{N}  \sum_{j=1}^N \log {q_\theta(\vy_j| \vx_i )} \Big]. \label{eq:vclub-sample-est}
\end{align}
%
% $ \hat{\MI}_{\text{vCLUB}} 
%      = \frac{1}{N} \sum_{n=1}^N  [ \log q_\theta(\vy_i  |\vx_i ) - \frac{1}{N}  \sum_{i=1}^N \log {q_\theta(\vy_i| \vx_i )} ].$
%
Using the variational approximation $q_\theta(\vy| \vx)$, vCLUB no longer guarantees a upper bound of $\MI(\vx;\vy)$. However, the vCLUB shares good properties with CLUB.  We claim that with good variational approximation  $q_\theta(\vy| \vx)$,  vCLUB can still hold a MI upper bound or become a reliable MI estimator. The following analyses support this claim.

Let $q_\theta(\vx,\vy) = q_\theta(\vy| \vx) p(\vx)$ be the variational joint distribution induced by $q_\theta(\vy| \vx)$. Generally, we have the following Theorem~\ref{thm:var-club-general}. Note that when $\vx$ and $\vy$ are independent, $\MI_{\text{vCLUB}}$ has exactly the same value as $\MI(\vx; \vy)$, without requiring any additional assumption on $q_\theta(\vy|
\vx)$. However, unlike in  Theorem~\ref{thm:club} as a sufficient and necessary condition, the ``independence between $\vx$ and $\vy$'' becomes sufficient but not necessary to conclude ``$\MI(\vx; \vy) = \vCLUB(\vx;\vy)$'', due to the variation approximation $q_\theta(\vy|\vx)$. 

\begin{theorem}\label{thm:var-club-general}
%Given a variational distribution $q(\vy| \vx)$, 
Denote $q_\theta(\vx,\vy) = q_\theta(\vy| \vx) p(\vx)$. If
\begin{equation*}
    \text{KL}\left(p(\vx, \vy) \Vert q_\theta(\vx,\vy)\right) \leq \text{KL} \left(p(\vx) p(\vy) \Vert q_\theta(\vx, \vy)\right),
\end{equation*} 
then $\MI(\vx;\vy) \leq \MI_{\text{vCLUB}}(\vx; \vy)$. The equality holds when $\vx$ and $\vy$ are independent. 
\end{theorem}

Theorem~\ref{thm:var-club-general} provides insight that vCLUB remains a MI upper bound if the variational joint distribution $q_\theta(\vx, \vy)$ is ``closer'' to $p(\vx, \vy)$ than to $p(\vx) p(\vy)$. Therefore, minimizing $\KL (p(\vx,\vy) \Vert q_\theta(\vx,\vy))$ will facilitate the condition in Theorem~\ref{thm:var-club-general} to be achieved.
%
%In practice, we can minimize the KL-divergence between $p(\vx,\vy)$ and $q_\theta(\vx,\vy)$ by maximizing the log-likelihood of $q_\theta(\vy|\vx)$ with samples $\{ \vx_i , \vy_i \}_{n = 1}^N$.
%
We show that $\KL (p(\vx,\vy) \Vert q_\theta(\vx,\vy))$ can be minimized by maximizing the log-likelihood of $q_\theta(\vy|\vx)$, because of the following equation:  %Note that $\KL(p(\vx,\vy)\Vert q_\theta(\vx,\vy)) = \KL(p(\vy| \vx) p(\vx) \Vert q_\theta(\vy| \vx) p(\vx)) = \KL(p(\vy| \vx) \Vert q_\theta(\vy| \vx))$
\begin{align}\textstyle
    &\min_{\theta} \KL (p(\vx, \vy) \Vert q_\theta(\vx,\vy)) \nonumber\\
    =& \min_{\theta} \bbE_{p(\vx,\vy)}[\log (p(\vy|\vx)p(\vx)) - \log (q_\theta(\vy| \vx) p(\vx))] \nonumber\\
    =& \min_{\theta} \bbE_{p(\vx,\vy)}[\log p(\vy| \vx)] - \bbE_{p(\vx,\vy)} [ \log q_\theta(\vy |\vx)]. \label{eq:kl-to-likelihood}
\end{align}
%which is equal to $\min_{\theta}\KL(p(\vy|\vx) \Vert q_\theta(\vy|\vx))$.
Equation~\eqref{eq:kl-to-likelihood} equals $\min_{\theta}\KL(p(\vy|\vx) \Vert q_\theta(\vy|\vx))$, in which
the first term has no relation with parameter $\theta$. Therefore,  $\min_{\theta} \KL (p(\vx,\vy) \Vert q_\theta(\vx,\vy))$ is equivalent to the maximization of the second term, $\max_{\theta} \bbE_{p(\vx,\vy)} [ \log q_\theta(\vy |\vx)]$.  With samples  $\{(\vx_i , \vy_i) \}_{i=1}^N$, we can maximize the log-likelihood function
$\calL(\theta):= \frac{1}{N}\sum_{i=1}^N \log q_\theta(\vy_i | \vx_i )$,
which is the unbiased estimation of  $\bbE_{p(\vx,\vy)} [ \log q_\theta(\vy |\vx)]$.

% we can minimizing $\KL(p(\vx,\vy) \Vert q_\theta(\vx,\vy)$ by maximize the log-likelihood function
% $  \calL(\theta)= \frac{1}{N}\sum_{n=1}^N \log q_\theta(\vy_i | \vx_i ), $
% which is the unbiased estimation of the second term $\bbE_{p(\vx,\vy)} [ \log q_\theta(\vy |\vx)]$.
In practice, the variational distribution $q_\theta(\vy|\vx)$ is usually implemented with neural networks. By enlarging the network capacity (\textit{i.e.}, adding layers and neurons) and applying gradient-ascent to the log-likelihood $\calL(\theta)$, we can obtain far more accurate approximation $q_\theta(\vy| \vx)$ to $p(\vy|\vx)$, thanks to the high expressiveness of neural networks~\citep{hu2019understanding,oymak2019overparameterized}. 
%
%which further reduces the KL-divergence between $p(\vy|\vx)$Assume that with sufficient model capacity \pengyu{citation},
%
Therefore, to further discuss the properties of vCLUB, we assume the neural network approximation $q_\theta$ achieves $\KL(p(\vy| \vx) \Vert q_\theta(\vy| \vx)) \leq \varepsilon$ with a small number $\varepsilon>0$. In the Supplementary Material, we quantitatively discuss the reasonableness of this assumption. Consider the KL-divergence between $p(\vx) p(\vy)$ and $q_\theta(\vx,\vy)$. If $\KL (p(\vx)p(\vy) \Vert q_\theta(\vx,\vy)) \geq \KL(p(\vx, \vy) \Vert q_\theta(\vx, \vy))$, by Theorem~\ref{thm:var-club-general}, vCLUB is already  a MI upper bound. Otherwise, if $\KL(p(\vx) p(\vy) \Vert q_\theta(\vx,\vy)) <\KL(p(\vx, \vy) \Vert q_\theta(\vx, \vy))$, we have the following corollary:
\begin{corollary}~\label{thm:club-as-estimation}
Given $\KL(p(\vy |\vx) \Vert q_\theta(\vy|\vx)) \leq \varepsilon$, if  $$ \KL(p(\vx, \vy) \Vert q_\theta(\vx, \vy))> \KL(p(\vx) p(\vy) \Vert q_\theta(\vx,\vy)),$$ then $\left| \MI(\vx;\vy) - \vCLUB(\vx; \vy) \right|< \varepsilon$.
\end{corollary}
Combining Corollary~\ref{thm:club-as-estimation} and Theorem~\ref{thm:var-club-general}, we conclude that with a good variational approximation $q_\theta(\vy |\vx)$, vCLUB can either remain a MI upper bound, or become a MI estimator whose absolute error is bounded by the approximation performance $\KL(p(\vy| \vx) \Vert q_\theta(\vy| \vx))$.

%  \begin{align}
%      \hat{\MI}_{\text{vCLUB}} 
%      = &\frac{1}{N^2} \sum_{n=1}^N \sum_{i=1}^N \left[\log q_\theta(\vy_i  |\vx_i ) - \log {q_\theta(\vy_i| \vx_i )} \right] \nonumber \\ 
%      = \frac{1}{N} \sum_{n=1}^N & [ \log q_\theta(\vy_i  |\vx_i ) - \frac{1}{N}  \sum_{i=1}^N \log {q_\theta(\vy_i| \vx_i )} ]. \label{eq:vclub-sample-est}
% \end{align}
%\vspace{-1.5mm}
\subsection{CLUB in MI Minimization}\label{sec:MI_minimizaiton}
%%\vspace{-1.5mm}

\begin{algorithm}[b]
%\small
%\SetAlFnt{\small}
\begin{algorithmic}
% \STATE {\bfseries Input:} { variational distribution $p_\sigma(\vx, \vy)$,  approximation network $q_\theta(\vy| \vx)$.}
\FOR{each training iteration}
 \STATE Sample $\{(\vx_i , \vy_i )\}_{i=1}^N$ from $p_\sigma(\vx,\vy)$\;
 \STATE Log-likelihood $\calL(\theta) = \frac{1}{N} \sum_{i=1}^N \log q_\theta(\vy_i  | \vx_i )$\;
 \STATE Update $q_\theta(\vy| \vx)$ by maximizing $\calL(\theta)$\;
 \FOR{$i = 1$ {\bfseries to} $N$}
 \IF {use \textit{sampling}}
 \STATE Sample $k'_i $ uniformly from $\{1,2,\dots,N \}$\;
 \STATE  $U_i  = \log q_\theta(\vy_i  | \vx_i ) - \log q_\theta( \vy_{k'_i }| \vx_{i}) $\; 
 \ELSE 
 \STATE $U_i  = \log q_\theta(\vy_i  | \vx_i ) - \frac{1}{N} \sum_{j=1}^N \log q_\theta( \vy_{j}| \vx_{i}) $ \;
  \ENDIF
 \ENDFOR
\STATE Update $p_\sigma(\vx,\vy)$ by minimize $\hat{\MI}_{\text{vCLUB}} = \frac{1}{N} \sum_{i=1}^N U_i $\;
 \ENDFOR
\end{algorithmic}
 \caption{MI Minimization with vCLUB}
 \end{algorithm}
% \end{minipage}

   \begin{figure*}[t]
        \centering
        \includegraphics[width = \linewidth]{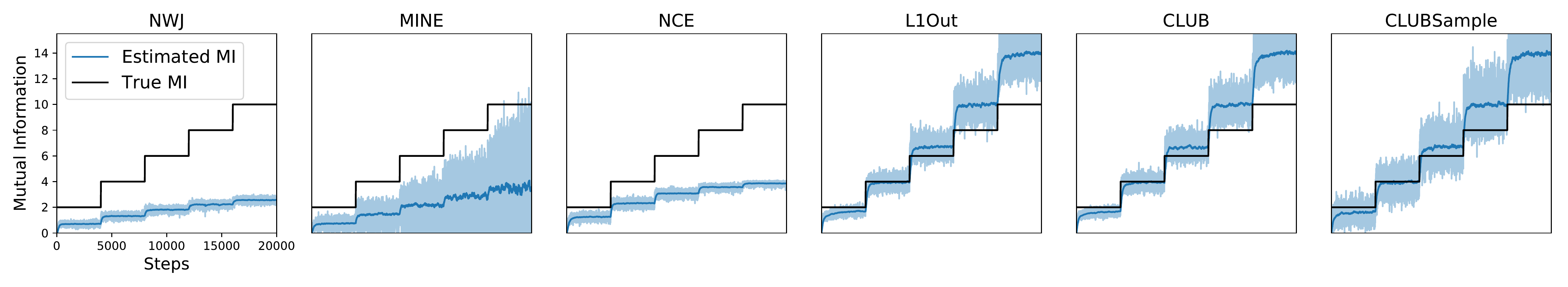}
        \includegraphics[width = \linewidth]{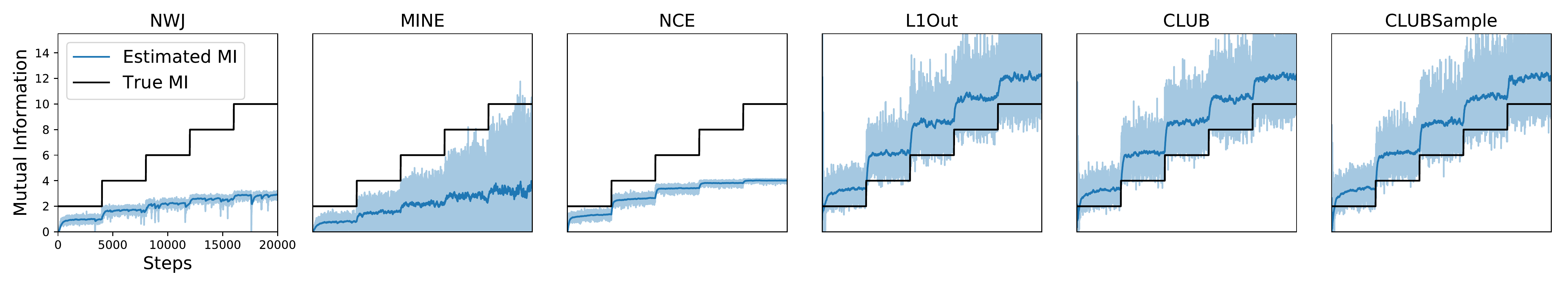}
        \vspace{-8mm}
     \caption{ Simulation performance of MI estimators. In the \textbf{top} row, data are from joint Gaussian distributions with the MI true value stepping over time. In the \textbf{bottom} row, a cubic transformation is further applied to  the Gaussian samples as $\vy$. In each figure,  the true MI values is  a step function shown as the black line. The estimated values are displayed as shadow blue curves. The dark blue curves shows the local averages of estimated MI, with a bandwidth equal to 200. }
        \label{fig:toy_MI_est}
\vspace{-3.mm}
\end{figure*}

One of the major applications of MI upper bounds is for mutual information minimization.
%The CLUB MI upper bound can be used as loss functions in MI minimization tasks. 
In general, MI minimization aims to reduce the correlation between two variables $\vx$ and $\vy$ by selecting an optimal parameter $\sigma$ of the joint variational distribution $p_\sigma(\vx, \vy)$. Under some application scenarios, additional conditional information between $\vx$ and $\vy$ is known. For example, in the information bottleneck task, the joint distribution between input $\vx$ and bottleneck representation $\vy$ is $p_\sigma(\vx, \vy) = p_\sigma(\vy | \vx) p(\vx)$. Then the MI upper bound $\MI_{\text{CLUB}}$ can be  calculated directly based on Eqn.~\eqref{eq:club-sample-est}.

For cases in which the conditional information between $\vx$ and $\vy$ remains unclear, we propose an MI minimization algorithm using the vCLUB estimator. At each training iteration, we first obtain a batch of samples $\{(\vx_i , \vy_i )\}$ from $p_\sigma(\vx , \vy)$. Then we update the variational approximation $q_\theta(\vy | \vx)$ by maximizing the log-likelihood $\calL(\theta) = \frac{1}{N} \sum_{i=1}^N \log q_\theta (\vy_i  | \vx_i )$. After  $q_\theta(\vy| \vx)$ is updated, we calculate the vCLUB estimator as described in~\eqref{eq:vclub-sample-est}. Finally, the gradient of $\hat{\MI}_{\text{vCLUB}}$ is calculated and back-propagated to parameters of $p_\sigma(\vx, \vy)$. The reparameterization trick~\citep{kingma2013auto} ensures the gradient back-propagates through the sampled embeddings $(\vx_i , \vy_i )$. Updating joint distribution $p_\sigma(\vx, \vy)$ will lead to the change of conditional distribution $p_\sigma(\vy | \vx)$. Therefore, we need to update the approximation network $q_\theta(\vy| \vx)$ again.
Consequently, $q_\theta(\vy| \vx)$ and $p_\sigma(\vx, \vy)$ are updated alternately during the training (as shown in Algorithm~1 without \textit{sampling}).

In each training iteration, the vCLUB estimator requires  calculation of all  conditional distributions $\{p_\sigma(\vy_j | \vx_i)\}_{i,j =1}^N$, which leads to $\mathcal{O}(N^2)$ computational complexity. To  further accelerate the calculate, for each positive sample pair $(\vx_i, \vy_i)$, instead of calculating the mean of the probabilities of  all negative pairs as $\frac{1}{N}\sum_{i=1}^N \log q_\theta(\vy_j | \vx_i)$ in \eqref{eq:vclub-sample-est},  we randomly sample a negative pair $(\vx_i, \vy_{k'_i})$ and use $\log q_\theta(\vy_{k'_i} | \vx_i)$ as an unbiased estimation, with $k'_i$ uniformly selected from indices $\{ 1, 2, \dots, N \}$. Then we obtain the sampled vCLUB (vCLUB-S) MI estimator:
\begin{equation}
 \hat{\MI}_{\text{vCLUB-S}}= \frac{1}{N} \sum_{i=1}^N \Big[\log q_\theta(\vy_i  |\vx_i ) -  \log q_\theta(\vy_{k'_i }| \vx_i)\Big],  \nonumber
\end{equation}
with the property of unbiasedness that $\bbE[\hat{\MI}_{\text{vCLUB-S}}] = \bbE[\hat{\MI}_{\text{vCLUB}}] = \MI_{\text{vCLUB}}(\vx;\vy)$.
By this sampling strategy, the computational complexity in each iteration can be reduced to  $\mathcal{O}(N)$ (as Algorithm 1 with \textit{sampling}). A similar sampling strategy can also be applied to CLUB when $p(\vy| \vx)$ is known.
Besides the acceleration, the vCLUB-S estimator  bridges the MI minimization with \textit{ negative sampling}, a commonly used training strategy for learning word embeddings (\textit{e.g.}, Word2Vec~\citep{mikolov2013distributed}) and node embeddings (\textit{e.g.}, Node2Vec~\citep{grover2016node2vec}), in which a positive data pair $(\vx_i, \vy_i)$ includes two nodes with an edge connection or two words in the same sentence, and  a negative pair $(\vx_i, \vy_{k'_i})$ is uniformly sampled from the whole graph or vocabulary.
%
%The mutual information is minimized by reducing the positive conditional probability, while enlarging the negative conditional probability. 
%
%
Although previous MI upper bounds also utilize the negative data pairs (such as L$\bm{1}$Out in \eqref{eq:leave-one-out}), they cannot hold an unbiased estimation when accelerated with the sampling strategy, because of the non-linear log function applied after the linear probability summation. The unbiasedness of our sampled CLUB thanks to the form of linear log-ratio summation.
 In the experiments, we find the sampled vCLUB estimator not only provides comparable MI estimation performance, but also improves the model  generalization abilities as a learning critic.
 %\pengyu{other good properties ? Robustness?}             

\vspace{-1.1mm}
\section{Experiments}\label{sec:experiments}
\vspace{-0.2mm}
 In this section, we first show the performance of CLUB as a MI estimator on tractable toy (simulated) cases, with samples drawn from Gaussian and Cubic distributions. Then we evaluate the  minimization ability of CLUB on two real-world applications: Information Bottleneck (IB) and Unsupervised Domain Adaptation (UDA). In the information bottleneck, the conditional distribution $p(\vy| \vx)$ is known, so we compare performance of both CLUB and variational CLUB (vCLUB) estimators and their sampled versions.
%
%In contrast, the toy experiments and the domain adaptation experiments mainly require the learning of variational approximation $q_\theta(\vy| \vx)$, in which we only test the performance of vCLUB with other MI bounds. 
In the  other experiments for which $p(\vy| \vx)$ is unknown, all the tested upper bounds require variational approximation. Without ambiguity, in experiments except the Information Bottleneck, we  abbreviate all variational  bounds (\textit{e.g.}, vCLUB) with their original names (\textit{e.g.}, CLUB) for simplicity.
%\subsection{Simulation Study}

\vspace{-1.mm}
\subsection{MI Estimation Quality}\label{sec:gaussian-est}
\vspace{-.5mm}
Following the setup from \citet{poole2019variational}, we apply CLUB as an MI estimator in two toy tasks: ($i$) estimating MI with samples $\{(\vx_i, \vy_i)\}$ drawn jointly from a multivariate Gaussian distribution with correlation $\rho$; ($ii$) estimating MI with samples $\{(\vx_i, (\mW\vy_i)^3)\}$, where  $(\vx_i, \vy_i)$ still comes from a Gaussian with correlation $\rho$, and $\mW$ is a full-rank matrix. Since the transformation $\vy \to (\mW \vy)^3$ is smooth and bijective, the mutual information is invariant~\citep{kraskov2004estimating}, $\MI(\vx ; \vy) = \MI(\vx; (\mW \vy)^3)$. 
For both of the tasks, the dimension of samples $\vx$ and $\vy$ is set to $d=20$. 
Under Gaussian distributions, the MI true value can be calculated as $\MI(\vx, \vy) = -\frac{d}{2} \log (1-\rho^2)$, and therefore we set the MI true value in the range $\{2.0,4.0,6.0,8.0,10.0\}$ by varying the value of $\rho$.  At each MI true value, we sample data batches 4000 times, with batch size equal to 64, for the training of variational MI estimators.

 \begin{figure}[t]
        \centering
       \includegraphics[width = 0.49\linewidth]{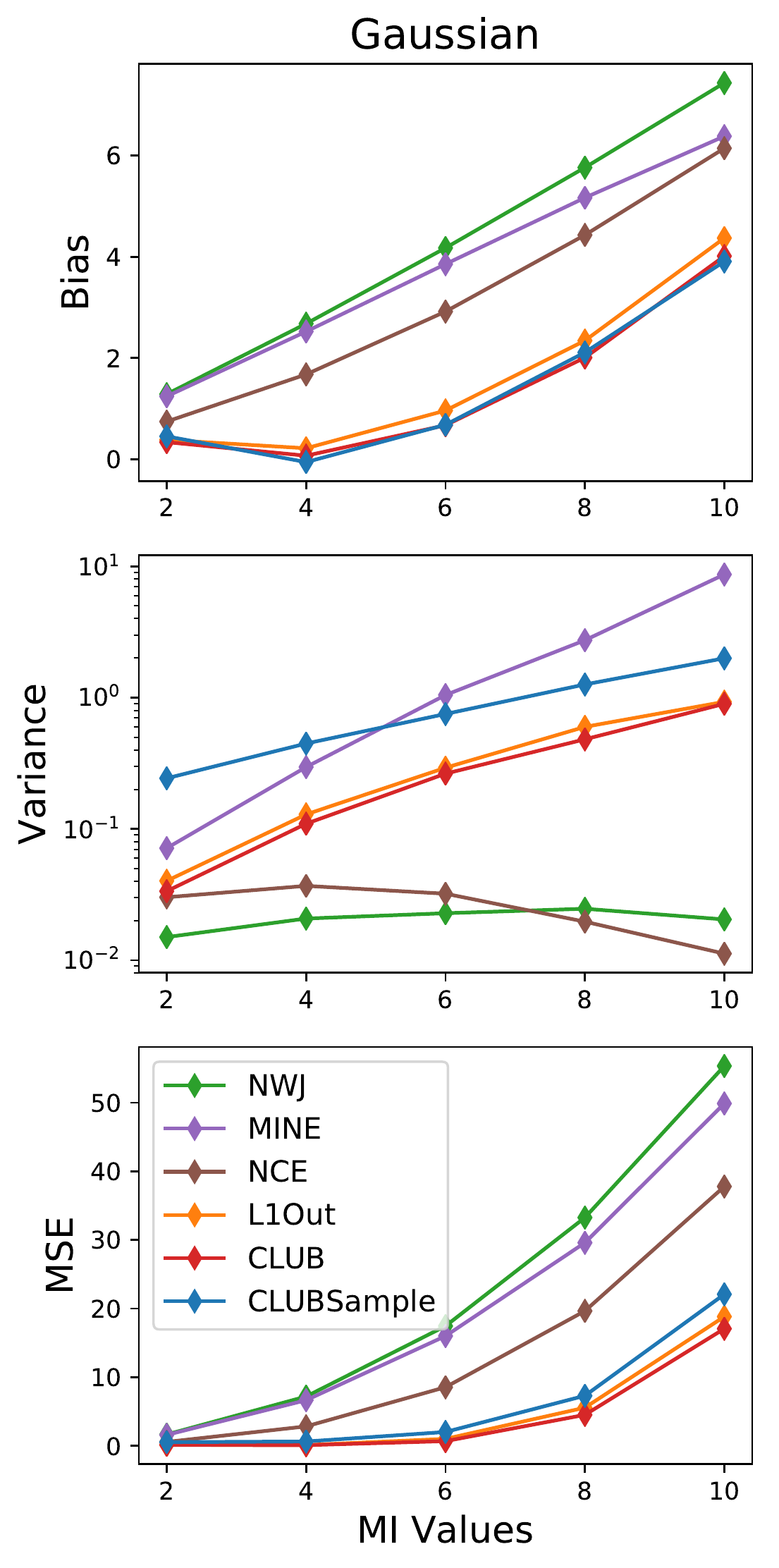} \includegraphics[width = 0.49\linewidth]{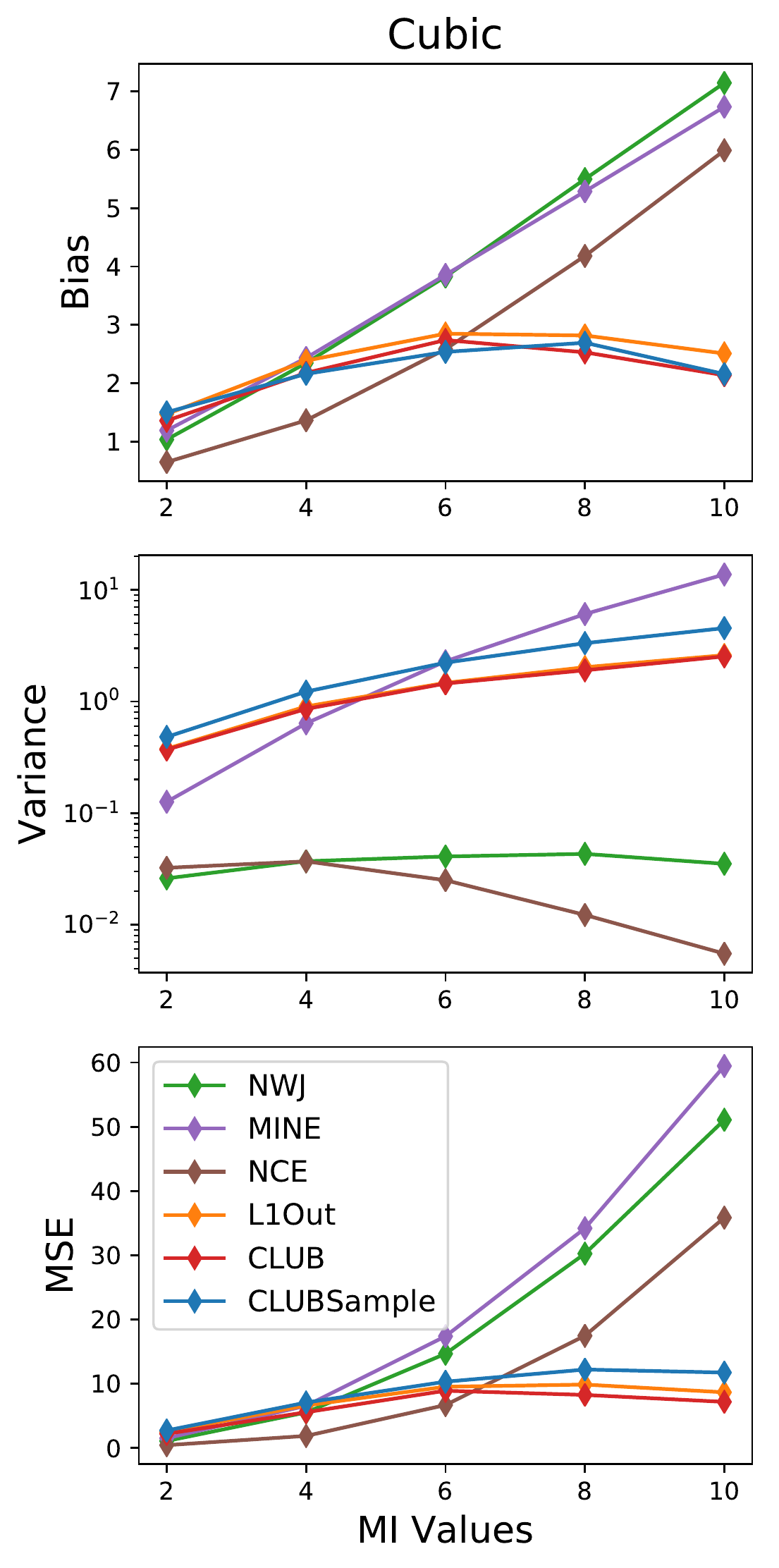}     
       \vspace{-3mm}
       \caption{ Estimation quality comparison of MI estimators. The \textbf{left} column shows the results of estimations under Gaussian distribution, while the \textbf{right} column is under Cubic setup. In each column, estimation metrics are reported as bias, variance, and mean-square-error (MSE). In each plot, the evaluation metric is reported with different true MI values varying from 2 to 10. }\label{fig:est_bias_variance}
     \vspace{-5.mm}
\end{figure}

We compare our method with baselines including MINE~\citep{belghazi2018mutual}, NWJ~\citep{nguyen2010estimating}, InfoNCE~\citep{oord2018representation}, VUB~\citep{alemi2016deep} and \loout~\citep{poole2019variational}. 
%In Section~\ref{sec:bounds}, details  to these estimators are described.
%
Since the conditional distribution $p(\vy| \vx)$ is unknown in this simulation setup, all upper bounds (VUB, \loout, CLUB) are calculated with an auxiliary approximation network $q_\theta(\vy| \vx)$. The approximation network has the same structure for all upper bounds,  parameterized in a Gaussian family, $q_\theta(\vy | \vx) = \calN(\vy | \bm{\mu}(\vx), \bm{\sigma}^2(\vx) \cdot \bm{\mathrm{I}})$ with mean $\bm{\mu}(\vx)$ and variance $\bm{\sigma}^2(\vx)$ inferred by neural networks. 
On the other hand, all the MI lower bounds (MINE, NWJ, InfoNCE) require learning of a value function $f(\vx,\vy)$. To make fair comparison, we set the value function and the neural approximation with one hidden layer and the same  hidden units. 
%For Gaussian setup, the number of hidden units is $20$; for Cubic setup, the number of hidden units is $40$. 
For both the  Gaussian and Cubic setups, the number of hidden units of our CLUB estimator is set to 15.
On the top of hidden layer outputs, we add the ReLU activation function.
The learning rate for all estimators is set to $5\times 10^{-3}$. 
%More setup details are provided in the Supplementary Material.
%
% The hidden unit size is set to 5. In the Supplementary Material, we provide a theoretical justification demonstrating that for good approximation $q_\theta(\vy| \vx)$, VUB and { \loout}  still remain MI upper bounds.  %The results of upper bounds in figure also support our theoretical founds.

We report in Figure~\ref{fig:toy_MI_est} the estimated MI values in each training step. The estimation of VUB has incomparably large bias, so we provide its results in the Supplementary Material.  Lower bound estimators, such as NWJ, MINE, and InfoNCE, provide estimated values mainly under the true MI values step function, while {\loout}, CLUB and Sampled CLUB~(CLUBSample) estimate values above the step function, which supports our theoretical analysis about CLUB with variational approximation. The numerical results of bias and variance in the estimation are reported in Figure~\ref{fig:est_bias_variance}. Among these methods, CLUB and CLUBSample have the lowest bias. The bias difference between CLUB and CLUBSample is insignificant, supporting our claim in Section~\ref{sec:MI_minimizaiton} that CLUBSample is an unbiased stochastic approximation of CLUB. {\loout} also provides small bias estimation which is slightly worse than CLUB.
NWJ and InfoNCE have the lowest variance under both setups. CLUBSample has larger variance than CLUB and {\loout} due to the use of the sampling strategy.
When considering the bias-variance trade-off as the mean square estimation error (MSE, equals bias$^2+$variance), CLUB outperforms other estimators, while {\loout} and CLUBSample also provide competitive performance. 

Although {\loout} estimator reaches similar estimation performance as our CLUB on toy examples, we find {\loout} fails to effectively reduce the MI when applied as a critic in real-world MI minimization tasks. The numerical results in   Section~\ref{sec:information-bottleneck} and Section~\ref{sec:domain-adaptation} support our claim.

%The Sampled CLUB also shows comparable estimation performance in the simulation, with much faster computational speed. \pengyu{ show speed comparison}.  

\vspace{-.8mm}
\subsection{Time Efficiency of MI Estimators}
\vspace{-.5mm}
Besides the estimation quality comparison, we 
further study the time efficiency of different MI estimators. We conduct the comparison under the same experimental setup as the Gaussian case in Section~\ref{sec:gaussian-est}. Each MI estimator is  tested with different batch size from 32 to 512. We count the total time cost of the whole estimation process and average it into each estimation step.
%
%how the negative sampling  in CLUBSample accelerates the computation, 
%
In Figure~\ref{fig:time_cost}, we report the average estimation time costs of different MI estimators. 
 MINE and CLUBSample have the highest computational efficiency; both have $\calO(N)$ computational complexity with respect to the sample size $N$, because of  the negative sampling strategy. Among other computational  $\calO(N^2)$ methods, CLUB has the highest estimation speed, thanks to its simple form as mean of log-ratios, which can be easily accelerated by matrix multiplication. Leave-one-out (L$\bm{1}$out) has the highest time cost, because it  requires ``leaving out'' the positive sample pair each time in the denominator of equation~\eqref{eq:leave-one-out}.

\begin{figure}[t!]
    \centering
    \includegraphics[width= 0.85\linewidth]{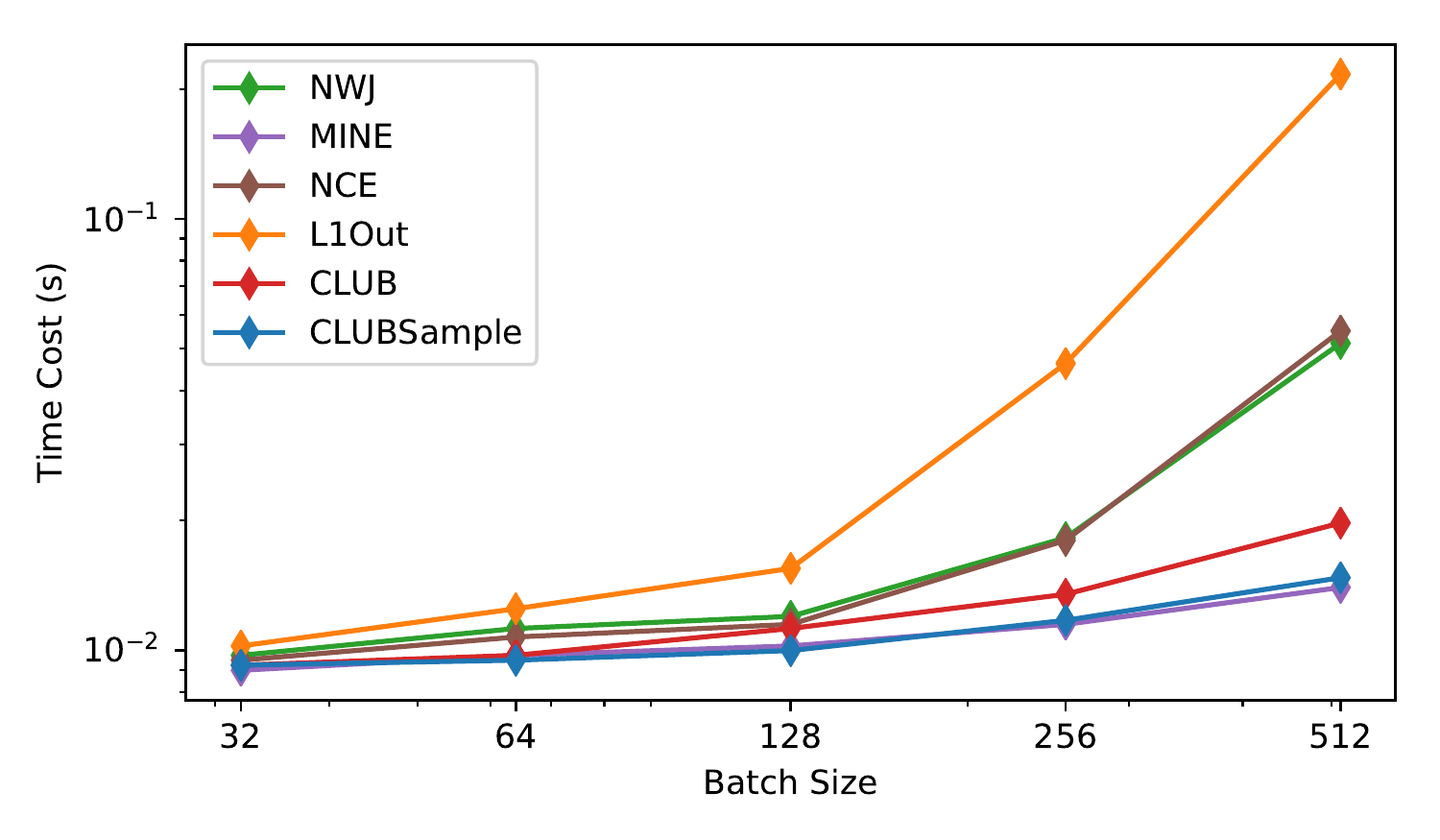}
    \vspace{-4mm}
    \caption{ Estimator speed comparison with different batch size. Both the axes have a logarithm scale.}
    \label{fig:time_cost}
    \vspace{-2.mm}
\end{figure}
\vspace{-0.6mm}
\subsection{MI Minimization in Information Bottleneck} \label{sec:information-bottleneck}
\vspace{-.2mm}

The Information Bottleneck~\citep{tishby2000information} (IB) is an information-theoretical method for  latent  representation learning. Given an input source $\vx \in \calX$ and a corresponding output target $\vy \in \calY$, the information bottleneck aims to learn an encoder $p_\sigma(\vz| \vx)$, such that the compressed latent code $\vz$ is highly relevant to the target $\vy$, with irrelevant source information from $\vx$ being filtered. In other words, IB seeks to find the sufficient statistics of $\vx$ with respect to $\vy$~\citep{alemi2016deep}, with minimum information used from $\vx$. To address this task, an objective is introduced as 
\begin{equation}\label{eq:IB-obj}
    \min_{p_\sigma(\vz| \vx)} - \MI(\vy; \vz) + \beta \MI(\vx; \vz)
\end{equation}
%\begin{equation}\label{eq:IB-obj}
  %\textstyle  \min_{p_\sigma(\vz| \vx)} - \MI(\vy; \vz) + \beta \MI(\vx; \vz),
%\end{equation}
where hyper-parameter $\beta >0$.
 Following the same setup from \citet{alemi2016deep}, we apply the IB technique in the permutation-invariant MNIST classification. The input $\vx$ is a vector converted from a $28\times 28$ image of a hand-written number, and the output $\vy$ is the class label of this number.
The stochastic encoder $p_\sigma(\vz|\vx)$ is implemented in a Gaussian variational family, $p_\sigma(\vz| \vx) =\calN(\vz| \mu_{\sigma}(\vx), \Sigma_{\sigma}(\vx)) $, where $\mu_\sigma$ and $\Sigma_\sigma$ are two fully-connected neural networks.

For the first part of the IB objective~\eqref{eq:IB-obj}, the 
MI between target $\vy$ and latent code $\vz$ is maximized. We use the same strategy as in  the deep variational information bottleneck (DVB)~\citep{alemi2016deep}, where  a variational classifier $q_\phi(\vy|\vz)$ is introduced to implement a Barber-Agakov  MI lower bound  (Eqn.~\eqref{eq:MI-lower-bound-BA}) of $\MI(\vy; \vz)$. 
The second term in the IB objective requires the MI minimization between input $\vx$ and the latent representation $\vz$. DVB~\citep{alemi2016deep} utilizes the MI variation upper bound (VUB) (Eqn.~\eqref{eq:var-upper-bound}) for the minimization of $\MI(\vx; \vz)$. 
Since the closed form of $p_\sigma(\vz| \vx)$ is already known as a Gaussian distribution parameterized by neural networks, we can directly apply our CLUB estimator for minimizing $\MI(\vx; \vz)$. Alternatively, the variational CLUB can be also applied under this scenario. Besides CLUB and vCLUB, we  compare  previous methods such as  MINE, NWJ, InfoNCE, and {\loout}. The misclassification rates for different MI estimators are reported in Table~\ref{tab:IB_MNIST}. 

\input{tables/IB_MNIST.tex}
 MINE achieves the lowest misclassification error among lower bound estimators.
Although providing good MI estimation in the Gaussian 
simulation study, {\loout} suffers from numerical instability in MI optimization and fails during training. Both CLUB and vCLUB estimators outperform previous methods in bottleneck representation learning, with lower misclassification rates. Note that  sampled versions of CLUB and vCLUB improve the accuracy compared with original CLUB and vCLUB, respectively, which verify the claim the negative sampling strategy improves model's generalization ability.
Besides, using variational approximation  $q_\theta(\vy|\vx)$ even attains higher accuracy than using ground truth $p_\sigma(\vy|\vx)$ for CLUB. Although $p_\sigma(\vy|\vx)$ provides more accurate MI estimation,  the variational approximation $p_\sigma(\vy|\vx)$ can add noise into the gradient of CLUB. Both the sampling and the variational approximation increase the randomness in the model,
which helps to increase the model generalization ability~\citep{hinton2012improving,belghazi2018mutual}.

%\textcolor{blue}{Jiachang: Possible reason why vCLUB performs better than CLUB: vCLUB involves a neural network to estimate $p(y\mid x)$. Because the neural network only gives an approximate answer, this introduces additional noise into the gradients, which could explain why vCLUB can get better results. However, during experiments, we observe that CLUB is more stable than vCLUB.}

%\vspace{-1.5mm}
\subsection{MI Minimization in Domain Adaptation}\label{sec:domain-adaptation}
%\vspace{-1.5mm}

Another important application of MI minimization is disentangled representation learning (DRL)~\citep{kim2018disentangling,chen2018isolating,locatello2019challenging}. Specifically, we aim to encode the data into several separate embedding parts, each with different semantic meanings. The semantically disentangled representations help improve the performance of deep learning models, especially in the fields of conditional generation~\citep{ma2018disentangled}, style transfer~\citep{john2019disentangled}, and domain adaptation~\citep{gholami2018unsupervised}.
To learn (ideally) independent disentangled representations, one  effective solution is to minimize the mutual information among different latent embedding parts.

\begin{figure}[t]
    \centering
    \includegraphics[width=0.9\linewidth]{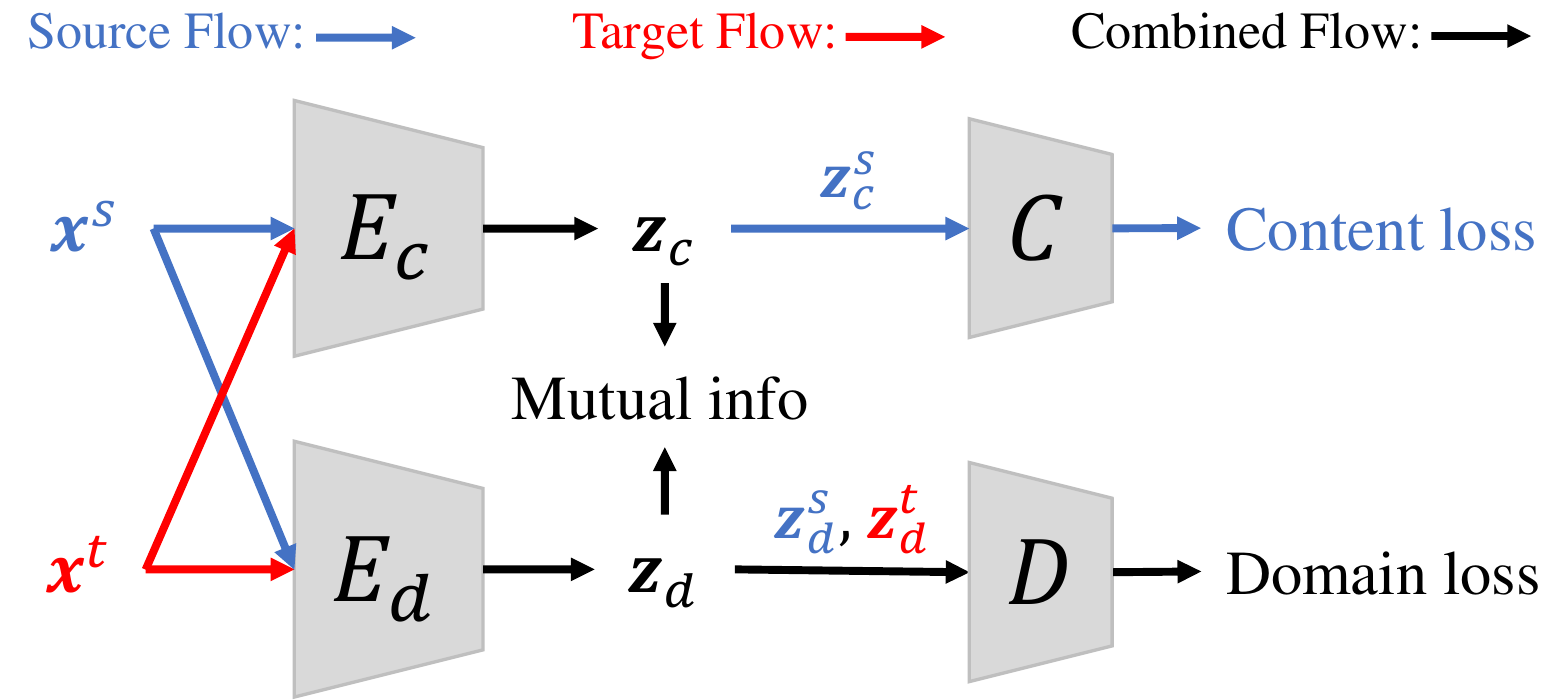}
    \vspace{-1mm}
    \caption{  The information-theoretical framework for unsupervised domain adaptation. The input data $\vx$ (including $\vx^s$ and $\vx^t$) are passed to a content encoder $E_c$ and a domain encoder $E_d$, with output feature $\vz_c$ and $\vz_d$, respectively. $C$ is the content classifier, and $D$ is the domain discriminator. The mutual information between $\vz_c$ and $\vz_d$ is minimized.}
    \label{fig:da_framework}
    \vspace{-2mm}
\end{figure}

We compare performance of MI estimators for learning disentangled representations in unsupervised domain adaptation (UDA) tasks.
In UDA, we have images $\vx^s \in \calX^s$  from the source domain $\calX^s$  and $\vx^t \in \calX^t$ from the target domain $\calX^t$. While each source image $\vx^s$ has a corresponding label $y^s$, no label information is available for observations in the target domain. The objective is to learn a model based on data $\{\vx^s, y^s \}$ and $\{\vx^t\}$, which not only performs well in source domain classification, but also provides satisfying predictions in the target domain.

To solve this problem, we use the information-theoretical framework inspired from \citet{gholami2018unsupervised}. Specifically, two feature extractors are introduced: the domain encoder $E_d$ and the content encoder $E_c$. The former encodes the domain information from an observation $\vx$ into a domain embedding $\vz_d = E_d(\vx)$; the latter outputs a content embedding $\vz_c = E_c(\vx)$ based on an input data point $\vx$. As shown in Figure~\ref{fig:da_framework}, the content embedding $\vz_c^s$ from the source domain is further used as an input to a content classifier $C(\cdot)$ to predict the corresponding class label, with a content loss defined as
$\mathcal{L}_c=\mathbb{E}[-y^s\log C(\vz_c^s)]$. The domain embedding $\vz_d$ (including $\vz_d^s$ and $\vz_d^t$) is input to a domain discriminator $D(\cdot)$ to predict whether the observation comes from the source domain or target domain, with a domain loss defined as $\mathcal{L}_d=\mathbb{E}_{\vx \in \calX^s}[\log D(\vz_d)]+\mathbb{E}_{\vx \in \calX^t}[\log (1-D(\vz_d))]$. Since the content information and the domain information should be independent, we minimize the mutual information $\MI(\vz_c, \vz_d)$ between the content embedding $\vz_c$ and domain embedding $\vz_d$. The final objective is (shown in Figure~\ref{fig:da_framework}): 
\begin{equation}\label{eq:UDA_obj}
    \min_{E_c, E_d, C, D}\MI(\vz_c, \vz_d) + \lambda_c \mathcal{L}_c + \lambda_d \mathcal{L}_d,
\end{equation}
% \begin{equation}
%     \min_{E_c, C}\mathcal{L}_c+\lambda\cdot \MI(\vz_c, \vz_d),\; \text{and} \min_{E_d, D}\mathcal{L}_d,
% \end{equation}
%
% \begin{equation}
%     \min_{E_c, E_d, C, D}\MI(\vz_c, \vz_d) + \lambda_c \mathcal{L}_c + \lambda_d \mathcal{L}_d,
% \end{equation}
%
where $\lambda_c, \lambda_d > 0$ are hyper-parameters. 
%We compare this MI-based DA framework with previous methods such as MMD, DANN, DSN, and MCD.\pengyu{add citation}

\input{tables/DA_MNIST}

We apply different MI estimators to the framework~\eqref{eq:UDA_obj}, and evaluate the performance on several DA benchmark datasets, including MNIST, MNIST-M, USPS, SVHN, CIFAR-10, and STL. Detailed description to the datasets and model setups is in the Supplementary Material. Besides the proposed information-theoretical UDA model, we also compare the performance with other UDA frameworks:   DANN~\cite{ganin2016domain}, DSN~\cite{bousmalis2016domain}, and MCD~\cite{saito2017maximum}.
The numerical results are shown in Table~\ref{tab:domain_adapt}. 
From the results, we find our MI-based disentangling shows competitive results with previous UDA methods. Among different MI estimators, the Sampled CLUB uniformly outperforms other competitive methods on four DA tasks. The stochastic sampling in CLUBSample improves the model generalization ability and preserves the model from overfitting. The other two MI upper bounds, VUB and {\loout}, fail to train a satisfying UDA model, whose results are  worse than the MI lower bound estimators. With {\loout}, the training loss cannot even decrease on the most challenging SVHN$\to$MNIST task, due to the numerical instability.
%
%Similarly, we could also derive another MI upper bound based on the conditional distribution $p(\vx | \vy)$. However, the neural network approximation to $p(\vx | \vy)$ would have worse performance relative to the approximation to $p(\vy| \vx)$. This is due to the fact that the dimension of $\vx$ is much higher than the dimension of $\vy$, so that $p(\vx| \vy)$ is a higher-dimensional distribution. Alternatively, the lower-dimensional distribution $p(\vy| \vx)$ used in our model is relatively easy to approximate with neural networks.


\vspace{-1.5mm}
\section{Conclusions}
\vspace{-.5mm}
We have introduced a novel mutual information upper bound called Contrastive Log-ratio Upper Bound (CLUB). This novel MI estimator can be extended to a variational version for general scenarios when only samples of the joint distribution are obtainable.
Based on the variational CLUB, we have proposed a  new MI  minimization algorithm, and further accelerated it with a negative sampling strategy.
We have studied the good  properties of CLUB both theoretically and empirically.  Experimental results on simulation studies and real-world applications show the attractive performance of CLUB  on both MI estimation and MI minimization tasks. This work provides an insight on the connection between mutual information and widespread machine learning training strategies, including contrastive learning and negative sampling. We believe
the proposed CLUB estimator will have vast applications for reducing the correlation of different model parts, especially in the domains of interpretable machine learning, controllable generation, and fairness. 

%Moreover, the CLUB MI estimator bridges the gap between MI minimization and contrastive learning. In future work, we plan to use this upper bound for other applications involving MI minimization, such as conditional generation, style transfer, and fairness. 

\section*{Acknowledgements} 
Thanks to Dongruo Zhou from UCLA for helpful discussions on network expressiveness. The portion of this work performed at Duke University was supported in part by DARPA, DOE, NIH, NSF and ONR.

\bibliography{references}
\bibliographystyle{icml2020}

\clearpage
\appendix
\onecolumn
\section{Proofs of Theorems}
\begin{proof}[Proof of Theorem~\ref{thm:var-club-general}]
We calculate the  gap between $\MI_{\text{vCLUB}}$ and $\MI(\vx ; \vy)$:
\begin{align*}
    \tilde{\Delta} := &\MI_{{\text{vCLUB}}} (\vx;\vy) - \MI(\vx; \vy)\nonumber \\
    = &  \bbE_{p(\vx,\vy)} [\log q_\theta (\vy|\vx)]  - \bbE_{p(\vx)}\bbE_{p(\vy)} [\log q_\theta (\vy| \vx)]   -  \bbE_{p(\vx, \vy)} \left[\log {p(\vy| \vx)} - \log{p(\vy)}\right] \nonumber \\
    = & \left[\bbE_{p(\vy)}[\log p(\vy)] - \bbE_{p(\vx) p(\vy)} [\log q_\theta(\vy| \vx)]\right] - \left[\bbE_{p(\vx, \vy)}[\log p(\vy| \vx)] - \bbE_{p(\vx, \vy)}[\log q_\theta(\vy| \vx)]\right] \\
    =& \bbE_{p(\vx)p(\vy)}[\log \frac{p(\vy)}{q_\theta (\vy| \vx)}] - \bbE_{p(\vx, \vy)} [\log \frac{p(\vy| \vx)}{q_\theta (\vy| \vx)}] \\
    = & \bbE_{p(\vx)p(\vy)}[\log \frac{p(\vx)p(\vy)}{q_\theta (\vy| \vx)p(\vx)}] - \bbE_{p(\vx, \vy)} [\log \frac{p(\vy| \vx)p(\vx)}{q_\theta (\vy| \vx)p(\vx)}]  \\
    = & \KL(p(\vx) p(\vy) \Vert q_\theta (\vx,\vy)) - \KL(p(\vx, \vy) \Vert q_\theta(\vx, \vy)).
\end{align*}
Therefore, $\vCLUB(\vx; \vy)$ is an upper bound of $\MI(\vx; \vy)$ if and only if $\KL(p(\vx) p(\vy) \Vert q_\theta (\vx,\vy)) \geq \KL(p(\vx, \vy) \Vert q_\theta(\vx, \vy))$.

If $\vx$ and $\vy$ are independent, $p(\vx) p(\vy) =p(\vx, \vy)$. Then, $\KL(p(\vx) p(\vy) \Vert q_\theta (\vx,\vy)) = \KL(p(\vx, \vy) \Vert q_\theta(\vx, \vy))$ and $\tilde{\Delta} = 0$. Therefore, $\vCLUB(\vx;\vy) = \MI(\vx; \vy)$, the equality holds.
\end{proof}

\begin{proof}[Proof of Corollary~\ref{thm:club-as-estimation}]
If $\KL(p(\vy| \vx) \Vert q_\theta(\vy| \vx)) \leq \epsilon$, then 
\begin{equation*}
    \KL(p(\vx, \vy) \Vert q_\theta(\vx, \vy)) = \bbE_{p(\vx,\vy)} [\log \frac{p(\vx,\vy)}{q_\theta(\vx,\vy)}] = \bbE_{p(\vx,\vy)} [\log \frac{p(\vy| \vx)}{q_\theta(\vy| \vx)}] = \KL(p(\vy| \vx) \Vert q_\theta(\vy| \vx)) \leq \epsilon.
\end{equation*}
By the condition $\KL(p(\vx, \vy) \Vert q_\theta(\vx,\vy) > \KL(p(\vx) p(\vy) \Vert q_\theta(\vx, \vy))$, we have $\KL(p(\vx) p(\vy) \Vert q_\theta(\vx, \vy))< \varepsilon$.

Note that the KL-divergence is always non-negative. From the proof of Theorem~\ref{thm:var-club-general}, 
\begin{align}
    \left|\vCLUB(\vx;\vy) - \MI(\vx;\vy)\right| =& \left| \KL(p(\vx) p(\vy) \Vert q_\theta (\vx,\vy)) - \KL(p(\vx, \vy) \Vert q_\theta(\vx, \vy)) \right| \nonumber \\ <& \max \left\{ \KL(p(\vx) p(\vy) \Vert q_\theta (\vx,\vy)), \KL(p(\vx, \vy) \Vert q_\theta(\vx, \vy)) \right\}  \leq \varepsilon,\nonumber 
\end{align}
which supports the claim.
\end{proof}

\section{Network Expressiveness in Variational Inference}\label{sec:neural-expressiveness}
In Section~\ref{sec:vCLUB}, when analyze the properties of the vCLUB estimator, we claim a reasonable assumption that with high expressiveness of the neural network $q_\theta(\vy|\vx)$, we can achieve $\KL(p(\vy| \vx) \Vert q_\theta (\vy| \vx)) < \varepsilon$. Here we provide a analysis under the scenario that the conditional distribution is  a Gaussian distribution, $p(\vy | \vx) = \calN(\vmu^*(\vx), \bm{\mathrm{I}})$. The variational approximation $q_\theta(\vy|\vx)$ is parameterized by $q_\theta(\vy| \vx) = \calN(\vmu_\theta(\vx), \bm{\mathrm{I}})$.

Then training samples pair $(\vx_i, \vy_i)$  can be treated as $(\vx_i, \vmu^*(\vx_i) + \bm{\xi}_i)$, where $\bm{\xi}_i \sim \calN(\bm{0}, \bm{\mathrm{I}})$. Then 
\begin{align*}
\log p(\vy| \vx) =&\log \prod_{d=1}^D[ \frac{1}{\sqrt{2 \pi}} e^{(y^{(d)} - \mu*^{(d)}(\vx))^2/2}] = -\frac{D}{2} \log (2\pi) - \frac{1}{2} \Vert \vy - \vmu^*(\vx)\Vert^2 , \nonumber\\
\log q_\theta(\vy| \vx) =&\log \prod_{d=1}^D[ \frac{1}{\sqrt{2 \pi}} e^{(y^{(d)} - \mu_\theta^{(d)}(\vx))^2/2}] = -\frac{D}{2} \log (2\pi) - \frac{1}{2} \Vert \vy - \vmu_\theta(\vx)\Vert^2 .
\end{align*}

The log-ratio between $p(\vy_i| \vx_i)$ and $q_\theta(\vy_i|\vx_i)$ is \begin{equation*}
    \log \frac{p(\vy_i|\vx_i)}{q_\theta(\vy_i| \vx_i)} = \log p(\vy_i| \vx_i) - \log q_\theta(\vy_i|\vx_i)  = [\vmu^*(\vx_i) - \vmu_\theta(\vx_i)]^T [\vy_i - \vmu_\theta(\vx_i) + \bm{\xi}_i]. 
\end{equation*}
We further assume $\Vert \vmu^*(\vx) - \vmu_\theta(\vx)\Vert < A$ is bounded. Then $|{\log p(\vy_i| \vx_i) } -{ \log q_\theta(\vy_i | \vx_i)}| < A \Vert \vy_i - \vmu_\theta(\vx_i) + \bm{\xi}_i \Vert$.

Denote a loss function $l(\vmu_\theta(\vx_i), \vy_i)= \Vert \vy_i - \vmu_\theta(\vx_i) + \bm{\xi}_i \Vert$. With all reasonable assumptions in \citet{hu2019understanding}, and applying the Theorem~5.1 in \citet{hu2019understanding}, we know that when the number of samples $n \to \infty$, the expected error $\bbE_{p(\vx, \vy)} [l(\vmu_\theta(\vx),\vy)] \to \infty$ with probability $1-\delta$.
\begin{equation*}
\KL(p(\vy| \vx) \Vert q_\theta(\vy| \vx)) = \bbE_{p(\vx,\vy)}[\log p(\vy| \vx) - \log q_\theta(\vy| \vx)]  < A \cdot \bbE_{p(\vx,\vy)} [l(\vmu_\theta(\vx) , \vy)].
\end{equation*}
 Therefore, when given a small number $\varepsilon >0$, having  the sample size $n$ large enough, we can guarantee that $\KL(p(\vy| \vx) \Vert q_\theta(\vy| \vx)) $ is smaller than $\varepsilon$.

\section{Properties of Variational Upper Bounds }
In the Section~\ref{sec:bounds}, we introduce two variational MI upper bounds with neural network approximation $q_\theta(\vy| \vx)$ to $p(\vy| \vx)$:
\begin{align*}
  \textstyle  \MI_{\text{vVUB}}(\vx; \vy) &= \bbE_{p(\vx,\vy)} \left[\log \frac{q_\theta(\vy|\vx)}{r(\vy)}\right], \\
\textstyle    \MI_{\text{v\loout}} (\vx; \vy)& = \bbE \left[ \frac{1}{N} \sum_{i=1}^N \left[\log \frac{q_\theta(\vy_i | \vx_i)}{\frac{1}{N-1} \sum_{j \neq i} q_\theta(\vy_i | \vx_j)}\right]\right].
\end{align*}
With the neural approximation $q_\theta(\vy| \vx)$, $\MI_{\text{vVUB}}$ and $\MI_{\text{v\loout}}$ no longer guarantee to be the MI upper bounds. However, both of the two estimators have good properties with  a good approximation $q_\theta(\vy| \vx)$.

\begin{theorem}\label{thm:vVUB}
If  $q_\theta(\vy|\vx)$ satisfies $\text{KL}(p(\vy| \vx) \Vert q_\theta(\vy| \vx)) \leq  \text{KL}(p(\vy) \Vert r(\vy)),$ then $\MI(\vx; \vy) \leq \MI_{\text{vVUB}}(\vx; \vy)$. 
\vspace{-2mm}
\end{theorem}
\begin{proof}[Proof of Theorem~\ref{thm:vVUB}]
With the conditional  $\text{KL}(p(\vy| \vx) \Vert q_\theta(\vy| \vx)) \leq  \text{KL}(p(\vy) \Vert r(\vy)),$
\begin{align*}
    \MI(\vx; \vy) = & \bbE_{p(\vx, \vy)} \left[\log \frac{p(\vy| \vx)}{p(\vy)}\right] 
    = \bbE_{p(\vx,\vy)} \left[ \log \left( \frac{p(\vy|\vx)}{q_\theta(\vy|\vx)} \cdot \frac{q_\theta(\vy|\vx)}{r(\vy)} \cdot \frac{r(\vy)}{p(\vy)} \right) \right] \\
    = &
    \bbE_{p(\vx,\vy)} \left[\log \frac{q_\theta(\vy|\vx)}{r(\vy)}\right] + \text{KL}(p(\vy| \vx) \Vert q_\theta(\vy| \vx)) -  \text{KL}(p(\vy) \Vert r(\vy))  \leq   \bbE_{p(\vx,\vy)} \left[\log \frac{q_\theta(\vy|\vx)}{r(\vy)}\right].
\end{align*}
\vspace{-4mm}
\end{proof}
\begin{theorem}\label{thm:vvCLUB}
Given $N-1$ samples $\vx_1,\vx_2,\dots,\vx_{N-1}$ from the marginal  $p(\vx)$,
If 
\begin{equation*}
    \text{KL}(p(\vy| \vx) \Vert  q_\theta(\vy| \vx)) \leq \bbE_{\vx_i \sim p(\vx)} \left[ \text{KL}\left(p(\vy) \Vert \frac{1}{N-1}\sum_{i=1}^{N-1} q_\theta(\vy| \vx_i)\right)\right],
\end{equation*} then $\MI(\vx; \vy) \leq \MI_{\text{v\loout}}(\vx; \vy)$. 
\vspace{-2mm}
\end{theorem}
\begin{proof}
Assume we have $N$ sample pairs $\{ (\vx_i,\vy_i)\}_{i=1}^N$ drawn from $p(\vx,\vy)$, then
\begin{align*}
&\MI(\vx; \vy) =\bbE_{(\vx_i, \vy_i) \sim p(\vx,\vy)} \left[ \frac{1}{N} \sum_{i=1}^N \left[\log \frac{p(\vy_i | \vx_i)}{p(\vy_i)}\right]\right]\\   
=&\bbE_{(\vx_i, \vy_i) \sim p(\vx,\vy)} \left[ \frac{1}{N} \sum_{i=1}^N \left[\log \left(\frac{p(\vy_i | \vx_i)}{q_\theta(\vy_i | \vx_i)} \cdot \frac{q_\theta(\vy_i|\vx_i)}{\frac{1}{N-1} \sum_{j \neq i} q_\theta(\vy_i | \vx_j)} \cdot \frac{\frac{1}{N-1} \sum_{j \neq i} q_\theta(\vy_i | \vx_j)}{p(\vy_i)}\right) \right]\right] \\
 =& \KL(p(\vy|\vx)\Vert q_\theta(\vy|\vx)) + \MI_{\text{vVUB}}(\vx; \vy) - \bbE \left[ \frac{1}{N} \sum_{i=1}^N \KL\left(p(\vy) \Vert\frac{1}{N-1} \sum_{j \neq i} q_\theta(\vy | \vx_j)\right)  \right].    
\end{align*}
Apply the condition in Theorem~\ref{thm:vvCLUB} to each $N-1$ combination of $\{\vx_j\}_{j \neq i}$, we conclude $\MI(\vx;\vy) \leq \MI_{\text{v\loout}}(\vx; \vy)$.
\vspace{-2mm}
\end{proof}
Theorem~\ref{thm:vVUB} and Theorem~\ref{thm:vvCLUB} indicate that if the approximation $q_\theta(\vy| \vx)$ is good enough, the estimators $\MI_{\text{vVUB}}$ and $\MI_{\text{v\loout}}$ can remain as MI upper bounds. Based on the analysis in Section~\ref{sec:neural-expressiveness}, when implemented with neural networks, the approximation can be far more accurate to preserve the variational estimators as MI upper bounds.

\section{Implementation Details}
\textbf{vCLUB with Gaussian Approximation}
When $q_\theta(\vy| \vx)$ is parameterized by $\calN (\vy | \vmu(\vx), \vsigma^2(\vx) \cdot \mathbf{I})$, then given samples $\{(\vx_i, \vy_i)\}_{i=1}^N$, we denote $\vmu_i = \vmu(\vx_i)$ and $\vsigma_i = \vsigma(\vx_i)$. Moreover, $\vmu_i = [\mu_i^{(1)}, \mu_i^{(2)}, \dots, \mu_i^{(D)}]^{\mathrm{T}}$, $\vsigma_i = [\sigma_i^{(1)}, \sigma_i^{(2)}, \dots, \sigma_i^{(D)}]^{\mathrm{T}}$, are $D$-dimensional vectors as $\vy_i = [y_i^{(1)}, y_i^{(2)}, \dots, y_i^{(D)}]^\mathrm{T}$. Then the conditional distribution
\begin{equation}
    q_\theta (\vy_j | \vx_i) =
    \prod_{d = 1}^D  (2 \pi (\sigma^{(d)}_i)^2)^{-1/2} \exp \left\{ -\frac{(y^{(d)}_j - \mu^{(d)}_i)^2}{2 (\sigma^{(d)}_i)^2 }\right\}.
\end{equation}
Therefore, the log-ratio 
\begin{align*}
   \log q_\theta(\vy_i | \vx_i) - \log q_\theta(\vy_j | \vx_i) 
= & \log \left( \prod_{d = 1}^D  (2 \pi (\sigma^{(d)}_i)^2)^{-1/2} \right) + \log \left( \prod_{d=1}^D \exp \left\{ -\frac{(y^{(d)}_i - \mu^{(d)}_i)^2}{2 (\sigma^{(d)}_i)^2 }\right\} \right)\\  &- \log \left( \prod_{d = 1}^D  (2 \pi (\sigma^{(d)}_i)^2)^{-1/2} \right) - \log \left( \prod_{d=1}^D \exp \left\{ -\frac{(y^{(d)}_j - \mu^{(d)}_i)^2}{2 (\sigma^{(d)}_i)^2 }\right\} \right) \\
= & \sum_{d=1}^D \left\{ -\frac{(y^{(d)}_i - \mu^{(d)}_i)^2}{2 (\sigma^{(d)}_i)^2 }\right\} - \sum_{d=1}^D \left\{ -\frac{(y^{(d)}_j - \mu^{(d)}_i)^2}{2 (\sigma^{(d)}_i)^2 }\right\} \\
=& -\frac{1}{2}  (\vy_i - \vmu_i)^{\mathrm{T}} \text{Diag}[\vsigma_i^{-2}](\vy_i - \vmu_i) + \frac{1}{2} (\vy_j - \vmu_i)^{\mathrm{T}} \text{Diag}[\vsigma_i^{-2}](\vy_j - \vmu_i),
\end{align*}
where $\text{Diag}[\vsigma_i^{-2}]$ is a $D \times D$ diagonal matrix with $(\text{Diag}[\vsigma_i^{-2}])_{d,d} = (\sigma_i^{(d)})^{-2}$, $d=1,2,\dots,D$.
The vCLUB estimator can be calcuated by
\begin{align*}
      \hat{\MI}_{\text{vCLUB}} 
     =& \frac{1}{N^2} \sum_{i=1}^N \sum_{j=1}^N \left[\log q_\theta(\vy_i  |\vx_i ) - \log {q_\theta(\vy_j| \vx_i )} \right] \nonumber \\
    % = &\frac{1}{N^2} \sum_{i=1}^N \sum_{j=1}^N \left[\sum_{d=1}^D \left\{ -\frac{(y^{(d)}_i - \mu^{(d)}_i)^2}{2 (\sigma^{(d)}_i)^2 }\right\} - \sum_{d=1}^D \left\{ -\frac{(y^{(d)}_j - \mu^{(d)}_i)^2}{2 (\sigma^{(d)}_i)^2 }\right\} \right] \\
    % = & \frac{1}{N} \sum_{i=1}^N \sum_{d=1}^D \left\{ -\frac{(y^{(d)}_i - \mu^{(d)}_i)^2}{2 (\sigma^{(d)}_i)^2 }\right\} - \frac{1}{N^2} \sum_{i=1}^N \sum_{j=1}^N  \sum_{d=1}^D \left\{ -\frac{(y^{(d)}_j - \mu^{(d)}_i)^2}{2 (\sigma^{(d)}_i)^2 }\right\} \\ 
     = & -\frac{1}{2} \left\{ \frac{1}{N} \sum_{i=1}^N (\vy_i - \vmu_i)^{\mathrm{T}} \text{Diag}[\vsigma_i^{-2}](\vy_i - \vmu_i)  \right\} + \frac{1}{2}\left\{ \frac{1}{N^2}\sum_{i=1}^N \sum_{j=1}^N(\vy_j - \vmu_i)^{\mathrm{T}} \text{Diag}[\vsigma_i^{-2}](\vy_j - \vmu_i) \right\}.
\end{align*}

\section{Detailed Experimental Setups}
%\textbf{MI estimation for Gaussian distribution:} All the MI lower bounds require learning of a value function $f(\vx,\vy)$; all the upper bounds require learning of a network approximation $q_\theta(\vy| \vx)$. To make fair comparison, we set the value function and the neural approximation with one hidden layer and the same  hidden units. For Gaussian setup, the number of hidden units is $20$; for Cubic setup, the number of hidden units is $40$. On the top of hidden layer outputs, we add the ReLU activation function.The learning rate for all estimators is set to $1\times 10^{-4}$. 

% \textbf{MI estimation for Bernoulli distribution:}
%  A two-dimensional joint Bernoulli distribution is  parameterized by$\left(\begin{array}{cc} p_{00} & p_{01} \\ p_{10} & p_{11} \end{array}\right)$, with $\sum_{i,j} p_{ij} = 1$ and $0\leq p_{ij} \leq 1$.
%  the mutual information between the two dimensions can be calculated by $\sum_{i,j} p_{ij} \log \frac{p_{ij}}{(p_{i0}+p_{i1})(p_{0j} + p_{1j})}$. We obtain the true parameter value for a target MI value by using gradient descent. The $p_{ij}$ values are implemented by the output of a Softmax function of $4$ logits. The neural approximation is implemented by two-layer MLP with hidden size $50$, The learning rate is $5\times 10^{-5}$. The batch size is $10$.

\textbf{Information Bottleneck:}
For the experiment on information bottleneck, we follow the setup from \citet{alemi2016deep}. The parameters $\mu_\sigma(\vx)$ and $\Sigma_\sigma(\vx)$ are the output from a MLP with layers $784 \to 1024 \to 1024 \to 2K$, where $K$ is the size of the bottleneck. We set $K=256$. For the variational classifier to implement the Barber-Agakov MI lower bound, the structure is set to a one-layer MLP. The batch size is 100. We set our learning rate to  $10^{-4}$, with an exponential decay rate of $0.97$ and a decay step of $1200$. 

\textbf{Domain Adaptation:}
The network is constructed as follows. 
Both feature extractors ($i.e.$, $E_c$ and $E_d$) are nine-layer convolutional neural network with leaky ReLU non-linearities. The content classifier $C$ and the domain discriminator $D$ are a one-layer and a two-layer MLPs, respectively. Images from each domain are normalized using Gaussian normalization. 

\begin{table}[hbt]
	\vskip 0.05in
	\centering
	\small
	%\centering
	\begin{tabular}{c|c|c}
		\toprule
		Classifier $C$ & Discriminator $D$ & Extractor (both $E_c$ and $E_d$) \\
		\midrule
		Content feature $\vz_c^s$ & Domain feature $\vz_d$ & Input data $\vx$  \\
		\midrule
		& & $3\times3$ conv. 64 lReLU, stride 1\\
	    & & $3\times3$ conv. 64 lReLU, stride 1\\
		& & $3\times3$ conv. 64 lReLU, stride 1\\
		& & $2\times2$ max pool, stride 2, dropout, $p = 0.5$, Gaussian noise, $\sigma=1$\\
		
		& & $3\times3$ conv. 64 lReLU, stride 1\\
		& MLP, 64 ReLU & $3\times3$ conv. 64 lReLU, stride 1\\
		MLP output $C(\vz_c^s)$ with shape 10 & MLP output $D(\vz_d)$ with shape 2 & $3\times3$ conv. 64 lReLU, stride 1\\                
		& & $2\times2$ max pool, stride 2, dropout, $p = 0.5$, Gaussian noise, $\sigma=1$\\
		
		& & $3\times3$ conv. 64 lReLU, stride 1\\
		& & $3\times3$ conv. 64 lReLU, stride 1\\
		& & $3\times3$ conv. 64 lReLU, stride 1\\
		& & global average pool, output feature with shape 64\\
		\bottomrule
	\end{tabular} 
%	\caption{Model architecture for the unsupervised domain adaptation experiments.}
	\label{Table:architecture}
\end{table}

\section{Numerical Results of MI Estimation}
We report the numerical results of MI estimation quality in Table~\ref{tab:MSE_MI_est}. The detailed setups are provided in Section~\ref{sec:gaussian-est}. Our CLUB estimator has the lowest estimation error when the ground-truth MI value goes larger.
\begin{table}[h]
  \centering

    \begin{tabular}{l|rrrrr|rrrrr}
    \toprule
               & \multicolumn{5}{c|}{Gaussian} &       \multicolumn{5}{c}{Cubic}   \\
    \midrule
    MI true value  & 2     & 4     & 6     & 8     & 10    & 2     & 4     & 6     & 8     & 10 \\
    \midrule
    VUB   & 3.85  & 15.33 & 34.37 & 61.25 & 95.70  & 2.09  & 10.38 & 25.56 & 47.84 & 77.59 \\
    NWJ   & 1.67  & 7.20   & 17.46 & 33.26 & 55.34 & 1.10   & 5.54  & 14.68 & 30.25 & 51.07 \\
    MINE  & 1.61  & 6.66  & 16.01 & 29.60  & 49.87 & 1.53  & 6.58  & 17.4  & 34.20  & 59.46 \\
    NCE   & 0.59  & 2.85  & 8.56  & 19.66 & 37.79 & 0.45  & 1.89  & 6.70   & 17.48 & 35.86 \\
    L1Out & \textbf{0.13}  & \textbf{0.11}  & 0.75   & 4.65  & 17.08 & 2.30   & 5.58  & 8.92  & 8.27  & 7.19 \\
    \midrule
    CLUB  & 0.15  & 0.12  & \textbf{0.70}  & \textbf{4.53}  & \textbf{16.57} & \textbf{2.22}  & \textbf{5.89}  & 8.25  & \textbf{8.23}  &\textbf{ 6.93} \\
    CLUBSample & 0.38  & 0.44  & 1.31  & 5.30   & 17.63 & 2.37  & \textbf{5.89}  & \textbf{8.07}  & 8.87  & 7.54 \\
    \bottomrule
    \end{tabular}%
     \caption{MSE of MI estimation}
  \label{tab:MSE_MI_est}%
\end{table}%

\end{document}

%% file: math_commands.tex
\usepackage{amsmath,amsthm,amsfonts,amssymb}
\usepackage{bm}
\usepackage{graphicx}
\usepackage{sidecap}
\usepackage{mathrsfs}
\usepackage{multirow, multicol}

\usepackage{wrapfig}
\usepackage{booktabs}

\usepackage{natbib} 

\newtheorem{theorem}{Theorem}[section]

\newcommand{\MI}{\mathrm{I}}

\newcommand{\vCLUB}{\MI_{\text{vCLUB}}}
\newcommand{\loout}{L$\bm{1}$Out}
\newcommand{\bbE}{\mathbb{E}}

\newcommand{\calL}{\mathcal{L}}

\newcommand{\calN}{\mathcal{N}}
\newcommand{\calO}{\mathcal{O}}

\newcommand{\calX}{\mathcal{X}}
\newcommand{\calY}{\mathcal{Y}}

\def\1{\bm{1}}

\def\vmu{{\bm{\mu}}}

\def\vsigma{{\bm{\sigma}}}

\def\vx{{\bm{x}}}
\def\vy{{\bm{y}}}
\def\vz{{\bm{z}}}

\def\mW{{\bm{W}}}

\DeclareMathAlphabet{\mathsfit}{\encodingdefault}{\sfdefault}{m}{sl}
\SetMathAlphabet{\mathsfit}{bold}{\encodingdefault}{\sfdefault}{bx}{n}

\newcommand{\KL}{\text{KL}}

%% file: tables/IB_MNIST.tex
\begin{table}[t]  \small
	\centering
	% \vskip 0.0in
	%\vspace{-1mm}
	% \def\arraystretch{1.0}
	% \begin{scriptsize}
	%\resizebox{\columnwidth}{!}{
	%\resizebox{0.85\columnwidth}{!}{

	\begin{tabular}{l|c}
 		\toprule%[1.2pt]
 
		\textbf{Method} & \textbf{Misclass. rate(\%)} \\
		\midrule
	%\hline
	  \textbf{NWJ} \citep{nguyen2010estimating}          &  1.29 \\
        \textbf{MINE} \citep{belghazi2018mutual}           & 1.17 \\
\textbf{InfoNCE} \citep{oord2018representation}           & 1.24 \\

        \midrule        	
        \textbf{DVB (VUB)} \citep{alemi2016deep}      & 1.13 \\
\textbf{\loout} \citep{poole2019variational}&  -    \\
    % \hline
    \midrule
        
        \textbf{CLUB}  & {1.12} \\
        \textbf{CLUB (Sample)} & 1.10 \\
        \textbf{vCLUB} & 1.10 \\
        \textbf{vCLUB (Sample)} & \textbf{1.06}\\
        %\textbf{vCLUBSample} & \\
        \bottomrule%[1.2pt]
	\end{tabular}
	\caption{ Performance on the Permutuation invariant MNIST classification.
	Different MI estimators are applied for the minimization of $\MI(\vx; \vz)$ in the Information Bottleneck. Misclassification rates of learned latent  representation $\vz$ are reported. The top three methods are MI lower bounds, while the rest are MI upper bounds. }
	\label{tab:IB_MNIST}
	\vspace{-3.5mm}
\end{table}

%% file: tables/DA_MNIST.tex
\setlength{\tabcolsep}{2.4pt}
\begin{table}[t]\small
    \centering
    \begin{tabular}{l|cccccc}
        \toprule
        \textbf{Method} & M$\rightarrow$MM & M$\rightarrow$U & U$\rightarrow$M & SV$\rightarrow$M & C$\to$S & S$\to$C\\
        \midrule
        \textbf{Source-Only} & 59.9 & 76.7 & 63.4 & 67.1 & - & -\\
        \midrule
        \multicolumn{7}{c}{MI-based Disentangling Framework} \\
        \midrule
        \textbf{NWJ} & 83.3 & 98.3 & 91.1 & 86.5& 78.2 & 71.0\\
        \textbf{MINE} & 88.4 & 98.1 & 94.8 & 83.4 &77.9 &70.5 \\
        \textbf{InfoNCE} & 85.5 & 98.3 & 92.7 & 84.1 &77.4&69.4 \\
        \midrule
        \textbf{VUB} & 76.4 & 97.1 & 96.3 & 81.5& -&-\\
        \textbf{L1Out} & 76.2 & 96.3 & 93.9 & - &77.8& 69.2 \\
        \textbf{CLUB} & 93.7 & \textbf{98.9} & 97.7 & 89.7 &78.7 & 71.8 \\
        \textbf{CLUB-S} & \textbf{94.6} & \textbf{98.9} & \textbf{98.1} & \textbf{90.6} & \textbf{79.1} &\textbf{72.3}\\
        \midrule
        \multicolumn{7}{c}{Other Frameworks} \\
        \midrule
        % \textbf{MMD} & 76.9 & - & - & 71.1 &- & - \\
        \textbf{DANN} & 81.5 & 77.1 & 73.0 & 71.1&- &- \\
        \textbf{DSN}  & 83.2 & 91.3 & - & 76.0 & - &-\\
        \textbf{MCD}  & 93.5 & 94.2 & 94.1 &  \textbf{92.6} & 78.1 & 69.2 \nocite{dai2019contrastively} \\
        \bottomrule
    \end{tabular}
    \caption{ Performance comparison on UDA. Datasets are MNIST (M), MNIST-M (MM), USPS (U),  SVHN (SV), CIFAR-10 (C), and STL (S). Classification accuracy on target domain is reported. Among results in MI-based disentangling framework, the top three  are MI lower bounds, while the rest are MI upper bounds. CLUB-S refers to Sampled CLUB.}
    \label{tab:domain_adapt}
  \vspace{-3.5mm}
\end{table}